\documentclass[10pt,conference,letterpaper]{IEEEtran}
\IEEEoverridecommandlockouts

\usepackage{amsmath,amsfonts,bm}

\def\eqref#1{equation~\ref{#1}}

\def\1{\mathbbm{1}}

\DeclareMathAlphabet{\mathsfit}{\encodingdefault}{\sfdefault}{m}{sl}
\SetMathAlphabet{\mathsfit}{bold}{\encodingdefault}{\sfdefault}{bx}{n}

\newcommand{\Var}{\mathrm{Var}}

\newcommand{\transpose}{\mathsf{T}}

\usepackage[colorlinks,linkcolor=blue,citecolor=blue]{hyperref} 
\usepackage{cite}
\usepackage{amsmath,amssymb,amsfonts}
\usepackage{algorithmic}
\usepackage[ruled,linesnumbered,vlined]{algorithm2e}
\usepackage{graphicx}
\usepackage{textcomp}
\usepackage{subfigure}
\usepackage{booktabs}
\usepackage{xcolor}
\usepackage{mathtools}
\usepackage{enumitem}
\usepackage{amsthm}
\usepackage{commath}
\usepackage{amsfonts}
\usepackage{nicefrac}
\usepackage{microtype}
\usepackage[utf8]{inputenc} 
\usepackage[T1]{fontenc} 
\usepackage{bbm}
\usepackage{bm}
\usepackage[capitalize]{cleveref}
\usepackage{nicefrac}
\usepackage{float}

\newtheorem{theorem}{Theorem}

\newtheorem{lemma}{Lemma}
\newtheorem{definition}{Definition}
\newtheorem{assumption}{Assumption}

\DeclareRobustCommand*{\IEEEauthorrefmark}[1]{\raisebox{0pt}[0pt][0pt]{\textsuperscript{\footnotesize #1}}}

\hypersetup{
    colorlinks=true,
    linkcolor=blue,
    filecolor=blue,      
    urlcolor=black,
    citecolor=blue,
}

\def\alg{{\texttt{MFPO}}}

\def\BibTeX{{\rm D\kern-.05em{\sc i\kern-.025em b}\kern-.08em
    T\kern-.1667em\lower.7ex\hbox{E}\kern-.125emX}}
    
\allowdisplaybreaks[4]

\begin{document}

\title{Momentum-Based Federated Reinforcement Learning with Interaction and Communication Efficiency
}

\author{
\IEEEauthorblockN{
Sheng Yue\IEEEauthorrefmark{1}, 
Xingyuan Hua\IEEEauthorrefmark{2}, 
Lili Chen\IEEEauthorrefmark{1}, 
Ju Ren\IEEEauthorrefmark{1}\IEEEauthorrefmark{3}\IEEEauthorrefmark{*}
\thanks{\textsuperscript{*}Corresponding author}}
\IEEEauthorblockA{
\IEEEauthorrefmark{1}Department of Computer Science and Technology, BNRist, Tsinghua University, Beijing, China \\
\IEEEauthorrefmark{2}School of Computer Science and Technology, Beijing Institute of Technology, Beijing, China \\ 
\IEEEauthorrefmark{3}Zhongguancun Laboratory, Beijing, China \\
\href{mailto:shengyue@tsinghua.edu.cn,dyh19@tsinghua.edu.cn,renju@tsinghua.edu.cn}{\texttt{\{shengyue,lilichen,renju\}@tsinghua.edu.cn}}, \href{mailto:xingyuanhua@bit.edu.cn}{\texttt{xingyuanhua@bit.edu.cn}}
}}

\maketitle

\begin{abstract}
Federated Reinforcement Learning (FRL) has garnered increasing attention recently. However, due to the intrinsic spatio-temporal non-stationarity of data distributions, the current approaches typically suffer from high interaction and communication costs. In this paper, we introduce a new FRL algorithm, named \alg{}, that utilizes momentum, importance sampling, and additional server-side adjustment to control the shift of stochastic policy gradients and enhance the efficiency of data utilization. We prove that by proper selection of momentum parameters and interaction frequency, \alg{} can achieve $\tilde{\mathcal{O}}(H N^{-1}\epsilon^{-3/2})$ and $\tilde{\mathcal{O}}(\epsilon^{-1})$ interaction and communication complexities ($N$ represents the number of agents), where the interaction complexity achieves linear speedup with the number of agents, and the communication complexity aligns the best achievable of existing first-order FL algorithms. Extensive experiments corroborate the substantial performance gains of \alg{} over existing methods on a suite of complex and high-dimensional benchmarks. 

\end{abstract}


\section{Introduction}
\label{sec:introduction}


With the rapid proliferation of Artificial Internet of Things (AIoT) applications and the increasing significance of data security, Federated Learning (FL) has emerged as a key enabler in the era of edge intelligence~\cite{mcmahan2017communication,xu2021edge,zhang2023towards}. Recently, to reconcile FL with ever-growing intelligent decision-making applications, 
there has been a surge of interest towards \emph{Federated Reinforcement Learning} (FRL), whereby distributed agents collaborate to build a decision policy with no need to share their raw trajectories~\cite{liu2019lifelong,nadiger2019federated,zhuo2020federated,fan2021fault,khodadadian2022federated}.
FRL has been deemed as a practically appealing approach to address the data hungry of Reinforcement Learning (RL)~\cite{khodadadian2022federated},
and demonstrated remarkable potential in a wide range of real-world systems, including robotics~\cite{liu2019lifelong}, autonomous driving~\cite{liang2019federated}, resource management in networking~\cite{yu2020deep}, and control of IoT devices~\cite{lim2020federated}.

However, the majority of current studies in FRL heuristically repurpose well-established supervised FL methods for the RL setting, e.g., directly combining \texttt{FedAvg} with classical \texttt{PG} or \texttt{Q-learning}~\cite{nadiger2019federated,lim2020federated,khodadadian2022federated}, neglecting a unique challenge embedded therein: \emph{the spatio-temporal non-stationarity of data distributions}. That is, in contrast to supervised FL operating on fixed datasets, FRL's intrinsic trial-and-error learning process typically necessitates each agent to explore the environment and sample new data using the current policy in each local update, causing continually varying data distributions \emph{across participating agents and training rounds}. As a result, it would inflict two major issues on the existing methods: on one hand, since interaction with real systems can be slow, expensive, or fragile, the simplistic combinations easily suffer from \emph{excessive interaction/sampling cost} during the continual environmental exploration; on the other hand, the dynamic data distributions can lead to substantially increased inter- and intra-agent shifts of  stochastic gradients, inducing \emph{brittle convergence properties and high communication complexity}. For instance, drawing upon the analysis for \texttt{FedAvg}~\cite{yang2021achieving}, we can show that the direct combination of \texttt{FedAvg} and \texttt{PG}~\cite{anwar2021multi} requires $\tilde{\mathcal{O}}(H\sigma^4_g\epsilon^{-2})$ environmental interactions and $\tilde{\mathcal{O}}(\sigma^4_g\epsilon^{-2})$ communication rounds to reach an $\epsilon$-stationary solution ($H$ and $\sigma^2_g$ represent the trajectory length and the potentially large variance of policy gradient, respectively), which can be quite expensive for resource-sensitive edge users. Of note, despite the existence of variance reduction techniques in supervised FL~\cite{johnson2013accelerating,karimireddy2020scaffold,khanduri2021stem}, the specific spatio-temporal non-stationarity of data distributions renders them inapplicable to FRL settings~\cite{fan2021fault}.

\begin{figure}[t]
    \centering
    \includegraphics[width=0.995\linewidth]{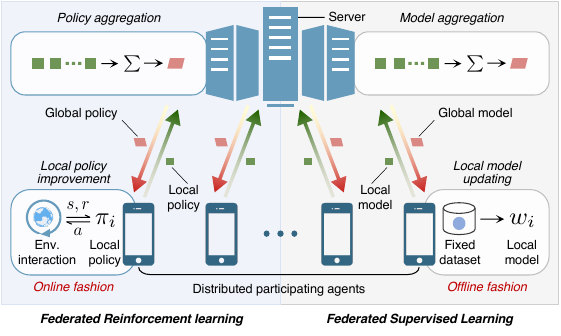}
    \vspace{-1.5em}
    \caption{Federated Reinforcement Learning vs Federated Supervised Learning.}  
    \vspace{-2em}
    \label{figure:system_model}
\end{figure}

\textbf{Contributions.} To overcome these challenges, this paper proposes  \emph{Momentum-assisted Federated Policy Optimization} (\alg{}), capable of jointly optimizing both the interaction and communication complexities. Specifically, 
we introduce a new FRL framework that utilizes momentum, importance sampling, and extra server-side adjustment to control the variates of stochastic policy gradients and improve the efficiency of data utilization. Building on this, we rigorously quantify the impacts of the inter- and intra-agent gradient errors on the performance. We prove that by proper selection of momentum parameters and interaction frequency, \alg{} can effectively counteract the gradient shifts and achieve $\tilde{\mathcal{O}}(H\sigma^2_g N^{-1}\epsilon^{-3/2})$  interaction complexity along with $\tilde{\mathcal{O}}(\sigma^2\epsilon^{-1})$ communication complexity ($N$ represents the number of agents). Notably, the interaction complexity achieves linear speedup with the number of agents, and the communication complexity recovers the best achievable across existing first-order stochastic FL algorithms. Finally, we evaluate \alg{} on a suite of complex and high-dimensional RL benchmarks, including Classic Control, MuJoCo, and image-based Atari games. The results demonstrate that \alg{} surpasses the state-of-the-art baseline methods by a significant margin, in terms of the performance and the efficiency of communication and interaction.




\section{Related Work}
\label{sec:related_work}


In recent years, several FRL algorithms have been proposed to address data-sharing constraints and facilitate safe co-training of policies~\cite{zhang2017survey,zhou2019edge,wang2023attrleaks}. Nadiger et al.~\cite{nadiger2019federated} propose an FRL approach that combines \texttt{DQN}~\cite{hasselt2010double} and \texttt{FedAvg}~\cite{mcmahan2017communication} to obtain personalized policies for individual players in the Pong game by employing the smoothing average technique. 
Lim et al.~\cite{lim2020federated} propose an FRL algorithm that combines Proximal Policy Optimization (\texttt{PPO})~\cite{schulman2017proximal} with \texttt{FedAvg}. Utilizing transfer learning, Liang et al.~\cite{liang2019federated} adapt Deep Deterministic Policy Gradient (\texttt{DDPG})~\cite{lillicrap2015continuous} for \texttt{FedAvg} to operate in autonomous driving scenarios. Cha et al.~\cite{cha2019federated} introduce a privacy-preserving variant of policy distillation, where a pre-arranged set of states and time-averaged policies is exchanged instead of raw data during training. Zhuo et al.~\cite{zhuo2020federated} propose a two-agent FRL framework in the discrete state-action spaces built on \texttt{Q-learning}~\cite{watkins1992q}, where agents share local encrypted Q-values and alternately update the global Q-network using multilayer perceptron (MLP). 
Anwar and Raychowdhury~\cite{anwar2021multi} study the adversarial attack issue in FRL via combining the policy gradient method with \texttt{FedAvg}, with the goal of training a unified policy for individual tasks. 
However, these works heuristically repurpose popular supervised FL methods for RL settings, not equipped with rigorous communication/interaction assurances. It remains a critical drawback due to the (potentially excessive) communication/interaction cost in real systems~\cite{daniel2019challenges}.

More recently, analogous to \cite{cha2019federated}, Khodadadian et al.~\cite{khodadadian2022federated} introduce an FRL algorithm by combining \texttt{FedAvg} with classical \texttt{Q-learning} and provide corresponding convergence guarantees, whereas they mainly concentrate on discrete cases (the state and action spaces are finite). Instead, this work operates in high-dimensional and continuous spaces. In a separate development, Fan et al.~\cite{fan2021fault} develop a fault-tolerant FRL algorithm, namely \texttt{FedPG-BR}, where a certain percentage (denoted by $\alpha\le0.5$) of the participating agents are subject to random system failures or adversarial attacks. \texttt{FedPG-BR} requires $\mathcal{O}(HN^{-2/3}\epsilon^{-5/3} + H\alpha^{4/3}\epsilon^{-5/3})$ interaction steps for each agent, under the assumption that the server can continually interact with the environment. In contrast, our algorithm offers a more favorable interaction complexity of $\tilde{\mathcal{O}}(HN^{-1}\epsilon^{-3/2})$ with the linear speedup and does not require any interaction between the server and the environment.


\section{Federated Reinforcement Learning}

\emph{Reinforcement Learning} (RL) is typically modeled as a \emph{Markov Decision Process} (MDP)  $\mathcal{M}\doteq\langle\mathcal{S},\mathcal{A},T,H,{r},\mu,\gamma\rangle$, consisting of state space $\mathcal{S}$, action space $\mathcal{A}$, transition dynamics $T:\mathcal{S}\times\mathcal{A}\rightarrow\mathcal{P}(\mathcal{S})$, episode horizon $H$, reward function ${r}:\mathcal{S}\times\mathcal{A}\rightarrow[0,R_{\max}]$, initial state distribution $\mu:\mathcal{S}\rightarrow[0,1]$, and  discount factor $\gamma\in(0,1]$~\cite{sutton2018reinforcement}. A (stationary stochastic) policy, $\pi_\theta(a|s)$, defines the probability of taking action $a$ at state $s$, which is parameterized by $\theta\in\mathbb{R}^d$. A trajectory, denoted as $\tau=\{s_1,a_1,\dots,s_H,a_H\}$, is a sequence of state-action pairs when rolling out $\pi$ with $M$. The objective of RL is to find a policy that can maximize the expected cumulative reward over the traversed trajectories:
\begin{align}
    \label{eq:objective}
    \max_{\theta}\mathbb{E}\Bigg[\sum^{H}_{h=1}\gamma^{h-1} r(s_h,a_h)\mid T,\mu,\pi_\theta\Bigg],
\end{align}
where the expectation is taken w.r.t. $s_1\sim\mu$, $a_h\sim\pi_\theta(\cdot|s_h)$, and $s_{h+1}\sim T(\cdot|s_h,a_h)$. Due to the intrinsic complexity (e.g., high-dimensional state-action spaces and delayed reward feedback), RL algorithms are often time-consuming and sample-inefficient.


\emph{Federated Reinforcement Learning} (FRL) 
solves Problem (\ref{eq:objective}) in a distributed fashion, where ${N}$ distributed agents federatively build a policy under the orchestration of a central server, without sharing their raw trajectories (as illustrated in \cref{figure:system_model}). The goal is to speed up policy search and improve sampling efficiency via collaboration among agents while complying with the requirement of information privacy or data confidentiality. 


\section{MFPO: Momentum-Based Federated Policy Optimization}

Denote the probability of $\tau=\{s_1,a_1,\dots,s_H,a_H\}$ under policy $\pi_\theta$ as $p(\tau|\theta)\doteq\mu(s_1)\prod^{H}_{h=1}T(s_{h+1}|s_{h},a_{h})\pi_\theta(a_{h}|s_{h})$, and the objective function as $J(\theta)\doteq-\mathbb{E}_{\tau\sim p(\cdot|\theta)}[r(\tau)]$ with $r(\tau)=\sum^H_{h=1}\gamma^{h-1}r(s_h,a_h)$ being the cumulative reward of trajectory $\tau$. Using the log-gradient trick and substituting the expression of $p(\tau|\theta)$, we can obtain the gradient of $J(\theta)$:
\begin{align}
    \nabla J(\theta)
    =-\mathbb{E}_{\tau\sim p(\cdot|\theta)}\Bigg[\bigg(\sum^H_{h=1}\nabla_\theta \pi_\theta(a_h|s_h)\bigg)r(\tau)\Bigg].
    \label{eq:policy_gradient}
\end{align}
Define $g(\theta;\tau)$ as an unbiased gradient estimator, which can be selected as widely used \texttt{REINFORCE}~\cite{williams1992simple} or \texttt{GPOMDP}~\cite{baxter2001infinite}. For example, the \texttt{REINFORCE} estimator can be expressed as $g(\theta;\tau)=\big(b-\sum^H_{h=1}r(s_h.a_h)\big)\sum^H_{h=1}\nabla_\theta\log\pi_\theta(a_h|s_h)$ with $b$ being the baseline reward.


A natural FRL solution is to integrate Policy Gradient (\texttt{PG}) directly into current supervised FL frameworks. Yet, due to the stochasticity exponential in $H$, the gradient estimates inevitably suffer from substantially high variance~\cite{agarwal2019reinforcement}. Besides, since the local policy updates will alter the agent-side distribution, $p(\cdot|\theta)$, on which $\nabla J(\theta)$ depends, the trajectory distributions across agents are spatio-temporally non-stationary. As a result, this sort of combination may cause pronounced inter- and intra-agent gradient shifts, significantly impeding learning performance.






Motivated by the recent advance of momentum-based distributed optimization~\cite{yu2019linear,khanduri2021stem}, we next introduce a novel momentum-assisted FRL algorithm, exploiting the techniques of momentum, importance sampling, and server-side adjustment, to tackle the above-mentioned problem. 
To be specific, the algorithm begins by initializing the policy parameters as $\theta^{(1)}_i=\bar{\theta}^{(1)}$ and then computes the corresponding initial directions as $\tilde{u}^{(1)}_i=(1/\tilde{D})\sum^{\tilde{D}}_{j=1}g(\theta^{(1)}_i;\tau^{(1)}_{i,j})$, with $\tilde{D}$ the number of trajectories generated from the initial policy. Subsequently, it alternates between local and global phases as follows.
	

\textbf{\emph{1) Local phase:}} In step $t$, each agent samples $D$ trajectories (denoted as $\tau^{(t)}_{i,j}$) via interacting with the environment using policy $\theta^{(t)}_i$. Then, it locally computes its updating direction by
\begin{align}
    \tilde{u}^{(t)}_i &=  \nu^{(t)}\bigg(\tilde{u}^{(t-1)}_i -  \frac{1}{D}\sum\nolimits^{D}_{j=1}w^{(t)}_{i,j}\cdot g(\theta^{(t-1)}_i;\tau^{(t)}_{i,j})\bigg)\nonumber\\
    &+\frac{1}{D}\sum\nolimits^{D}_{j=1}g(\theta^{(t)}_i;\tau^{(t)}_{i,j}),
    \label{eq:local_direction}
\end{align}
with $\nu^{(t)}\in[0,1]$ being the momentum parameter and $w^{(t)}_{i,j}$ the importance weight, computed as
\begin{align}
    w^{(t)}_{i,j}=w(\theta^{(t)}_i,\theta^{(t-1)}_i;\tau^{(t)}_{i,j})=\frac{\prod^{H}_{h=1}\pi^{(t-1)}_i(a_{h,i}|s_{h,i})}{\prod^{H}_{h=1}\pi^{(t)}_i(a_{h,i}|s_{h,i})}.
\end{align}
If $t~\mathrm{mod}~{K} \neq 0$, agents update their policy parameters locally:
\begin{align}
    \label{eq:local_update}
    \theta^{(t+1)}_i = \theta^{(t)}_i - \alpha^{(t)}\tilde{u}^{(t)}_i,\quad\forall i\in[{N}],
\end{align}
with $K$ the number of local steps and $\alpha^{(t)}$ the stepsize.

\textbf{\emph{2) Global phase:}} If $t~\mathrm{mod}~{K} = 0$, agents upload their local parameters and directions to the server for aggregation:
\begin{align}
    \bar{u}^{(t)} = \frac{1}{N}\sum\nolimits^{N}_{i=1}\tilde{u}^{(t)}_i, \quad
    \bar{\theta}^{(t)} = \frac{1}{N}\sum\nolimits^{N}_{i=1}\theta^{(t)}_i.
    \label{eq:global_aggregation}
\end{align}
The server carries out server-side adjustment as follows:
\begin{align}
    \bar{\theta}^{(t+1)}=\bar{\theta}^{(t)} - \alpha^{(t)}\bar{u}^{(t)}.
    \label{eq:global_momentum}
\end{align}
Then, it distributes the parameter and direction to all agents:
\begin{align}
    \label{eq:synchronazation}
    \theta^{(t+1)}_i = \bar{\theta}^{(t+1)}, \quad \tilde{u}^{(t)}_i=\bar{u}^{(t)},
\end{align}
and starts the next round.

We term our algorithm \emph{Momentum-assisted Federated Policy Optimization} (\alg{}) and outline the pseudocode in Alg.~\ref{alg:mfpo}. Intuitively, the momentum term along with importance sampling can keep track of past gradients in an off-policy manner, capable of improving sample efficiency while reducing the effect of fluctuations in the intra-agent gradient estimates~\cite{yu2019linear}. On the other hand, the server-side adjustment enables the global policy to continue moving along the dominant dimension and hence alleviating the inter-agent gradient shift. Next, we show how to select momentum parameters and interaction frequency to jointly optimize interaction and communication complexities.

\begin{algorithm}[ht]
    \LinesNumbered
    Initialize policy parameters and updating directions\;
    \For{$t=1$ \KwTo $T$}{
        \For{$i=1$ \KwTo $N$}{
        \tcp{Local phase}
            Agent $i$ rolls out local policy with environment, generates $D$ trajectories, and computes local direction by \cref{eq:local_direction}\;
            \If{$t~\mathrm{mod}~{K} \neq 0$}{
                Agent $i$ updates local policy by \cref{eq:local_update}\;
            }
        }
        \If{$t~\mathrm{mod}~{K} = 0$}{
            \tcp{Global phase}
            Agents upload local policies and directions\;
            Server updates global policy by \cref{eq:global_aggregation,eq:global_momentum}\;
            Server distributes global policy to all agents\;
        }
    }
    \caption{\alg{}}
    \label{alg:mfpo}
\end{algorithm}


\section{Theoritical Analysis}
\label{sec:theory}

In this section, we analyze the performance of \alg{}. We first introduce the necessary assumptions and then present our main result, followed by detailed proofs. 

\subsection{Notations and Assumptions}

For $\tau~\mathrm{mod}~K\neq 0$ , we define auxiliary variables as follows:
\begin{align}
    \label{eq:auxiliary}
    \bar{u}^{(\tau)} \doteq \frac{1}{N}\sum\nolimits^{N}_{i=1}\tilde{u}^{(\tau)}_i,\quad
    \bar{\theta}^{(\tau)} \doteq \frac{1}{N}\sum\nolimits^{N}_{i=1}\theta^{(\tau)}_i,
\end{align}
and for each $i\in[N]$, we define $\tilde{u}^{(0)}_i\doteq0$. When clear from the context, we use $\alpha_t$ and $\nu_t$ instead of $\alpha^{(t)}$ and $\nu^{(t)}$ for conciseness. We denote $t_q\doteq qK$ with $q\in\{0,\dots,M\}$ the index of communication rounds, and denote $[n] = \{1,2,\cdots,n\}$ for any $n\in\mathbb{N}$. In addition, we represent $\tilde{L}\doteq\max\{\tilde{L}_g,L\}$ where $L$ and $\tilde{L}_g$ are defined in \cref{lem:successive_difference,lem:error_contraction} respectively.

Due to the non-concavity of Problem (\ref{eq:objective})~\cite{agarwal2019reinforcement}, it is generally not feasible to measure the optimality by function values. Instead, the convergence of non-convex problems is typically characterized via finding an $\epsilon$-first-order stationary point ($\epsilon$-FOSP), defined as follows.
\begin{definition}
A solution $\theta\in\mathbb{R}^{d}$ is called an $\epsilon$-first-order stationary point ($\epsilon$-FOSP) of Problem (\ref{eq:objective}), if $\Vert\nabla J(\theta)\Vert^2\le\epsilon$. 
\end{definition}

We impose two commonly used assumptions in analyzing policy gradient methods as follows~\cite{papini2018stochastic,xu2020improved,fan2021fault}. 

\begin{assumption}
\label{asmp:derivative_bounds}
There exist $\beta_1,\beta_2>0$ such that for all $s\in\mathcal{S},a\in\mathcal{A}$, the log-density of the policy function satisfies
\begin{align}
    \|\nabla_\theta\log\pi_\theta(a|s)\|\le\beta_1,\quad \|\nabla^2_\theta\log\pi_\theta(a|s)\|\le\beta_2.
\end{align}
\end{assumption}

\begin{assumption}
\label{asmp:bounded_variance}
There exists $\sigma_g,\sigma_w>0$ such that the following fact holds:
\begin{align}
    &\mathbb{E}_{\tau\sim p(\cdot|\theta)}\big[\|g(\theta;\tau)-\nabla J(\theta)\|^2\big]\le \sigma_g^2,\quad\forall \theta\in\mathbb{R}^d,\\
    \label{eq:bounded_variance_w}
    &\Var(w(\theta,\theta';\tau))\le\sigma^2_w,\quad\forall \theta,\theta'\in\mathbb{R}^d,\tau\sim p(\cdot|\theta),
\end{align}
where $w(\theta,\theta';\tau)=p(\tau|\theta')/p(\tau|\theta)$ is the importance weight used in the algorithm.
\end{assumption}

Assumptions \ref{asmp:derivative_bounds} and \ref{asmp:bounded_variance} bound the gradient of the policy log-density and the variance of the gradient estimator, respectively. We suppose Assumptions \ref{asmp:derivative_bounds} and \ref{asmp:bounded_variance} and $T=MK$ ($M\in\mathbb{N}$) hold throughout this section. 

\subsection{Main Results}

We define the \emph{communication complexity} as the total number of communication rounds necessary for the algorithm to reach an $\epsilon$-FOSP, and the \emph{interaction complexity} as the total number of actions that each agent requires taking in the environment to achieve the $\epsilon$-FOSP. We define $c_\alpha$, $c_\nu$, and $c_t$ as follows:
\begin{align}
    c_t\doteq\max\left\{\frac{c^3_\nu c^3_\alpha}{2^{12}K^3{\tilde{L}}^3},2^{12}K^3D^2N^2\sigma^2_g - \sigma^2_g t,2\sigma^2_g\right\}\nonumber\\
    c_\nu\doteq\frac{{\tilde{L}}^2}{24 K(DN)^2}+\frac{64{\tilde{L}}^2}{DN},~c_\alpha\doteq\frac{(DN\sigma_g)^{2/3}}{{\tilde{L}}},
    \label{eq:parameters}
\end{align}
where positive integers $N$, $K$, and $D$ represent the numbers of agents, local updates, and trajectories required in each local update, respectively. Our main result is presented below.

\begin{theorem}
\label{thm:convergence}
Suppose the stepsizes and momentum parameters are selected as $\alpha_t = c_\alpha/(c_t+\sigma^2_g t)^{1/3}$ and $\nu_{t+1}=1 - c_\nu \alpha^2_{t}$. For any $\lambda\in[0,1]$, if $D=\mathcal{O}((T/N^2)^{1/2-\lambda/2})$, $\tilde{D}=DK$, and $K=\mathcal{O}((T/N^2)^{\lambda/3})$, \alg{} finds an $\epsilon$-FOSP after at most $\tilde{\mathcal{O}}(\epsilon^{-1})$ communication rounds and $\tilde{\mathcal{O}}(HN^{-1}\epsilon^{-3/2})$ environmental interactions.
\end{theorem}
\begin{proof}
The result can be obtained by substituting the expressions of $K$, $D$ and $\tilde{D}$ in  \cref{eq:eq31}. We omit it for brevity.
\end{proof}

\noindent\textbf{\emph{Remarks.}} \cref{thm:convergence} indicates that \alg{} achieves $\tilde{\mathcal{O}}(\epsilon^{-1})$ and $\tilde{\mathcal{O}}(HN^{-1}\epsilon^{-3/2})$ communication and interaction complexities by appropriate selection of the momentum parameters and the interaction frequency. The communication complexity recovers the best achievable of existing first-order FL algorithms~\cite{drori2020complexity,khanduri2021stem}. The interaction complexity exhibits linear speedup with the number of agents, making it superior to current FRL methods~\cite{fan2021fault}. It implies that \alg{} can effectively cope with the gradient shifts and the interaction cost caused by the spatio-temporal non-stationary data distributions. In addition, \cref{thm:convergence} reveals a tradeoff between the local updates, $K$, and the required trajectories per step, $D$, characterized by $\lambda\in[0,1]$. This means with a large number of local updates, the required trajectories per step can be set relatively small, and vice versa. In practical terms, this flexibility in adjusting $K$ and $D$ allows \alg{} to adapt to different scenarios and requirements.

\subsection{Detailed Proofs}

In this subsection, we detail the proof for \cref{thm:convergence}. We begin by bounding the \emph{successive difference} of the objective function in the following lemma.

\begin{lemma}
\label{lem:successive_difference}
For $t\in(t_q,t_{q+1}]$, the following fact holds true:
\begin{align}
    \mathbb{E}\big[J(\bar{\theta}^{(t+1)})\big]  &\le \mathbb{E}\big[J(\bar{\theta}^{(t)})\big] - \frac{\alpha_t}{2}\mathbb{E}\big[\|\nabla J(\bar{\theta}^{(t)}) \|^2\big] \nonumber\\
    &- \frac{\alpha_t - \alpha^2_t L}{2}\mathbb{E}\big[\|\bar{u}^{(t)}\|^2\big] + \alpha_t\mathbb{E}\big[\|\bar{\varepsilon}^{(t)}\|^2\big]\nonumber\\
    &+\frac{\alpha_tL^2}{N}\sum\nolimits^N_{i=1}\mathbb{E}\big[\|\theta^{(t)}_i-\bar{\theta}^{(t)}\|^2\big],
    \label{eq:successive_difference}
\end{align}
with $\bar{\varepsilon}^{(t)} \doteq \bar{u}^{(t)}-(1/N)\sum^{N}_{i=1}\nabla J(\theta^{(t)}_i)$ being the gradient error and $L\doteq HR_{\max} (H\beta^2_1+\beta_2)/(1-\gamma)$.
\end{lemma}
\begin{proof}
Built upon \cref{asmp:derivative_bounds} and \cite[Proposition 5.2]{xu2020improved}, $J(\theta)$ is $L$-smooth, which implies 
\begin{align}
    \label{eq:smoothness}
    J(\theta_1) \le\;& J(\theta_2) + \nabla J(\theta_2)^\transpose (\theta_1 - \theta_2)+ \frac{L}{2}\|\theta_1 - \theta_2\|^2.
\end{align}
Based on the $L$-smooth and \cref{eq:local_update,eq:auxiliary}, we have
\begin{align}
    J(\bar{\theta}^{(t+1)}) 
    =\;& J(\bar{\theta}^{(t)}) -\alpha_t \nabla J(\bar{\theta}^{(t)})^\transpose \bar{u}^{(t)}+ \frac{\alpha^2_t L}{2}\|\bar{u}^{(t)}\|^2\nonumber\\
    =\;& J(\bar{\theta}^{(t)}) - \frac{\alpha_t}{2}\mathbb{E}\big[\|\nabla J(\bar{\theta}^{(t)}) \|^2\big]- \frac{\alpha_t - \alpha^2_t L}{2}\nonumber\\
    &\cdot\mathbb{E}\big[\|\bar{u}^{(t)}\|^2\big]+\frac{\alpha_t}{2} \underbrace{\|\bar{u}^{(t)}-\nabla J(\bar{\theta}^{(t)})\|^2}_{(a)},
    \label{eq:eq3}
\end{align}
where second equality is derived by adding and subtracting $\alpha_t\|\bar{u}^{(t)}\|^2$ and utilizing $2\theta_1^\transpose \theta_2=\|\theta_1\|^2 + \|\theta_2\|^2 - \|\theta_1-\theta_2\|^2$. Regarding $(a)$, we have
\begin{align}
    (a) &= \bigg\|\bar{u}^{(t)}-\sum^N_{i=1}\frac{\nabla J(\theta^{(t)}_i)}{N}+\sum^N_{i=1}\frac{\nabla J(\theta^{(t)}_i)}{N}-\nabla J(\bar{\theta}^{(t)})\bigg\|^2\nonumber\\
    &\le 2\|\bar{\varepsilon}^{(t)}\|^2+\frac{2}{N}\sum^N_{i=1}\big\|\nabla J(\theta^{(t)}_i)-\nabla J(\bar{\theta}^{(t)})\big\|^2\label{eq:eq34}\\
    &\le 2\|\bar{\varepsilon}^{(t)}\|^2+\frac{2L^2}{N}\sum^N_{i=1}\|\theta^{(t)}_i-\bar{\theta}^{(t)}\|^2,\tag{the smoothness}
\end{align}
where derived \cref{eq:eq34} using the following relationship:
\begin{align}
    \label{eq:squared_sum}
    \|\theta_1+\theta_2+\cdots+\theta_n\|^2\le n\|\theta_1\|^2 + \cdots + n\|\theta_n\|^2.
\end{align}
Plugging $(a)$ in \cref{eq:eq3} and taking expectations on both sides yield the result.
\end{proof}

Next, we bound the last term of \cref{eq:successive_difference} in \cref{lem:error_accumulation}.

\begin{lemma}
    \label{lem:error_accumulation}
    For $t\in(t_q,t_{q+1}]$ and $i\in[N]$, we have
    \begin{align}
        \mathbb{E}\big[\|\theta^{(t)}_i - \bar{\theta}^{(t)}\|^2\big]\le\left(K-1\right) \sum^{t}_{{\tau}=t_{q}+1}\alpha^2_{{\tau}}\cdot\mathbb{E}\big[\|\tilde{u}^{({\tau})}_i-\bar{u}^{({\tau})}\|^2\big].
        \label{eq:error_accumulation}
    \end{align}
\end{lemma}
\begin{proof}
    When $t=t_q+1$, $\mathbb{E}\big[\|\theta^{(t_q+1)}_i - \bar{\theta}^{(t_q+1)}\|^2\big] = 0$ holds due to \cref{eq:synchronazation}.
    When $t\in(t_q+1,t_{q+1}]$, it follows
    \begin{align}
        \mathbb{E}\big[\|\theta^{(t)}_i - \bar{\theta}^{(t)}\|^2\big]
        =\mathbb{E}\Bigg[\bigg\|\sum^{t-1}_{{\tau}=t_q+1} \alpha_{{\tau}}\Big(\tilde{u}^{({\tau})}_i -  \bar{u}^{({\tau})}\Big)\bigg\|^2\Bigg]\tag{from \cref{eq:local_update} and $\theta^{(t_q+1)}_i = \bar{\theta}^{(t_q+1)}$}\\
        \label{eq:eq2}
        \le (K-1)\sum^{t-1}_{{\tau}=t_q+1}\alpha^2_{{\tau}}\cdot\mathbb{E}\big[\|\tilde{u}^{({\tau})}_i-\bar{u}^{({\tau})}\|^2\big],
    \end{align}
    where the last inequality is derived via $t_{q+1}-t_q=K$ and \cref{eq:squared_sum}. Thus, we complete the proof.
\end{proof}

Recall the definition of $\bar{\varepsilon}^{(t)}$ in \cref{lem:successive_difference}. \cref{lem:error_accumulation} characterizes the error accumulation in the interates of Alg.~\ref{alg:mfpo}. Substituting \cref{eq:error_accumulation} in \cref{eq:successive_difference}, we obtain
\begin{align}
	\mathbb{E}\big[J(\bar{\theta}^{(t+1)})\big] &\le \mathbb{E}\big[J(\bar{\theta}^{(t)})\big] - \frac{\alpha_t}{2}\mathbb{E}\big[\|\nabla J(\bar{\theta}^{(t)}) \|^2\big] \nonumber\\
	&- \frac{\alpha_t - \alpha^2_t L}{2}\mathbb{E}\big[\|\bar{u}^{(t)}\|^2\big] + \alpha_t\underbrace{\mathbb{E}\big[\|\bar{\varepsilon}^{(t)}\|^2\big]}_{\mathclap{\textrm{Gradient error}}}\nonumber\\
	&+ \frac{(K-1)L^2\alpha_t}{N}\sum\nolimits^{t-1}_{{\tau}=t_{q}+1}\alpha^2_{{\tau}}\nonumber\\
	&\cdot\underbrace{\sum\nolimits^N_{i=1}\mathbb{E}\big[\|\tilde{u}^{({\tau})}_i-\bar{u}^{({\tau})}\|^2\big]}_\textrm{Gradient consensus error}.
    \label{eq:successive_difference_2}
\end{align}
It suggests that the expected descent of $J(\cdot)$ relies on both the expected \emph{gradient error} and the expected \emph{gradient consensus error}. For conciseness, we denote the gradient consensus error as $\delta_t\doteq\sum^N_{i=1}\mathbb{E}[\|\tilde{u}^{(t)}_i-\bar{u}^{(t)}\|^2]$. In what follows, we bound the two errors respectively.


Regarding the gradient error, we introduce \cref{lem:error_contraction} to show how it contracts over time. 
\begin{lemma}
\label{lem:error_contraction}
Denote constants $L_g\doteq H\beta_2(R_{\max}+b)/(1-\gamma)$, $G_g\doteq  H\beta_1(R_{\max}+b)/(1-\gamma)$, $c_w\doteq H(2H\beta^2_1 + \beta_2)(\sigma^2_w + 1)$ and $\tilde{L}_g\doteq \sqrt{2(L^2_g+G^2_g c_w)}$ where $b$ is the baseline reward. Then, for $t\in[T]$, the following fact holds:
\begin{align}
    \mathbb{E}\big[\|\bar{\varepsilon}^{(t)}\|^2\big] &\le\nu^2_t\cdot\mathbb{E}\big[\|\bar{\varepsilon}^{(t-1)}\|^2\big] + \frac{4\tilde{L}_g^2\nu^2_t\alpha^2_{t-1}}{DN}\mathbb{E}\big[\|\bar{u}^{(t-1)}\|^2\big]\nonumber\\
    &+\frac{2\sigma^2_g(1-\nu_t)^2}{DN}+ \frac{8(K-1)\tilde{L}_g^2\nu^2_t\alpha^2_{t-1}\delta_{t-1}}{KD N^2}
\end{align}
with $\bar{\varepsilon}^{(t)}$ being the gradient error defined in \cref{lem:successive_difference}.
\end{lemma}
\begin{proof}
	From the definition of $\bar{\varepsilon}^{(t)}$ and \cref{eq:auxiliary}, we have
	\begin{align}
		\mathbb{E}\big[\|\bar{\varepsilon}^{(t)}\|^2\big]
		&= \frac{1}{D^2 N^2}\sum^N_{i=1}\sum^{D}_{j=1}\mathbb{E}\bigg[\Big\|g(\theta^{(t)}_i;\tau^{(t)}_{i,j})-\nabla J(\theta^{(t)}_i)\nonumber\\
		&- \nu_t\Big( w^{(t)}_{i,j} g(\theta^{(t-1)}_i;\tau^{(t)}_{i,j})-\nabla J(\theta^{(t-1)}_i)\Big)\Big\|^2\bigg]\nonumber\\
		\label{eq:eq5}
		&+  \nu^2_t\mathbb{E}\big[\|\bar{\varepsilon}^{(t-1)}\|^2\big],
	\end{align}
	where we add and substract $(1/N)\sum^N_{i=1}\nu_t\nabla J(\theta^{(t-1)}_i)$, expand the norms, and use the fact that the corresponding cross terms are zero, which can be easily verified via the tower rule and the unbiasedness of the importance-weighted gradient estimator $w^{(t)}_{i,j} g(\theta^{(t-1)}_i;\tau^{(t)}_{i,j})$. That is, for any $\theta,\theta'\in\mathbb{R}^d$, the following holds:
	\begin{align} 
		\nabla J(\theta)&= \mathbb{E}_{\tau\sim p(\cdot|\theta)}\left[g(\theta;\tau)\right]\tag{the unbiasedness of $g(\theta;\tau)$}\\
		&= \int p(\tau|\theta')\cdot\frac{p(\tau|\theta)}{p(\tau|\theta')}\cdot g(\theta;\tau)\dif\tau\nonumber\\
		&=\mathbb{E}_{\tau\sim p(\cdot|\theta')}\left[w(\theta',\theta;\tau) g(\theta;\tau)\right].
	\end{align}
	For the first term in \cref{eq:eq5}, we have
	\begin{align}
		&\mathbb{E}\Big[\big\|g(\theta^{(t)}_i;\tau^{(t)}_{i,j})-\nabla J(\theta^{(t)}_i)\nonumber\\
		&- \nu_t\big( w^{(t)}_{i,j} g(\theta^{(t-1)}_i;\tau^{(t)}_{i,j})-\nabla J(\theta^{(t-1)}_i)\big)\big\|^2\Big]\nonumber\\
		=\;& 2\nu^2_t\mathbb{E}\Big[\big\|g(\theta^{(t)}_i;\tau^{(t)}_{i,j})-w^{(t)}_{i,j} g(\theta^{(t-1)}_i;\tau^{(t)}_{i,j})\nonumber\\
		&- \big(\nabla J(\theta^{(t)}_i)-\nabla J(\theta^{(t-1)}_i)\big)\big\|^2\Big]+2(1-\nu_t)^2\nonumber\\
		&\cdot\mathbb{E}\Big[\big\| g(\theta^{(t)}_i;\tau^{(t)}_{i,j})-\nabla J(\theta^{(t)}_i)\big\|^2\Big]\tag{from \cref{eq:squared_sum}}\nonumber\\
		\le\;& 2\nu^2_t\mathbb{E}\Big[\big\|g(\theta^{(t)}_i;\tau^{(t)}_{i,j})-w^{(t)}_{i,j} g(\theta^{(t-1)}_i;\tau^{(t)}_{i,j})\big\|^2\Big]+2\sigma^2_g\nonumber\\
		&\cdot(1-\nu_t)^2\tag{\cref{asmp:bounded_variance}, mean variance inequality}\nonumber\\
		\le\;& 4\nu^2_t\mathbb{E}\Big[\big\|g(\theta^{(t)}_i;\tau^{(t)}_{i,j})-g(\theta^{(t-1)}_i;\tau^{(t)}_{i,j})\big\|^2\Big]\nonumber\\ &+4\nu^2_t\mathbb{E}\Big[\big\|(1-w^{(t)}_{i,j}) g(\theta^{(t-1)}_i;\tau^{(t)}_{i,j})\big\|^2\Big]+2(1-\nu_t)^2\sigma^2_g\tag{adding and substracting $g(\theta^{(t-1)}_i;\tau^{(t)}_{i,j})$ and using \cref{eq:squared_sum}}\\
		\overset{(a)}{\le}\;& 4\nu^2_tL^2_g \mathbb{E}\big[\|\theta^{(t)}_i-\theta^{(t-1)}_i\|^2\big]+4\nu^2_tG^2_g\mathbb{E}\big[(1-w^{(t)}_{i,j})^2\big]\nonumber\\
		&+2(1-\nu_t)^2\sigma^2_g\tag{from \cref{eq:smooth_bounded_g}}\\
		\overset{(b)}{\le}\;& 4\nu^2_tL^2_g \mathbb{E}\big[\|\theta^{(t)}_i-\theta^{(t-1)}_i\|^2\big]+4\nu^2_tG^2_g c_w\mathbb{E}\big[\|\theta^{(t)}_i\nonumber\\
		&-\theta^{(t-1)}_i\|^2\big]+2(1-\nu_t)^2\sigma^2_g\tag{from \cref{eq:importance_mean_var}}\\
		=\;& 2\nu^2_t\tilde{L}_g^2\mathbb{E}\big[\|\theta^{(t)}_i-\theta^{(t-1)}_i\|^2\big]+2(1-\nu_t)^2\sigma^2_g\label{eq:eq8}\\
		\le\;& 2\alpha^2_{t-1}\nu^2_t\tilde{L}_g^2 \mathbb{E}\big[\|\tilde{u}^{(t-1)}_i\|^2\big]+2(1-\nu_t)^2\sigma^2_g\tag{from \cref{eq:local_update}}\\
		\le\;&2(1-\nu_t)^2\sigma^2_g+ 4\alpha^2_{t-1}\nu^2_t\tilde{L}_g^2\mathbb{E}\big[\|\tilde{u}^{(t-1)}_i-\bar{u}^{(t-1)}\|^2\big]\nonumber\\
		&+4\alpha^2_{t-1}\nu^2_t\tilde{L}_g^2\mathbb{E}\big[\|\bar{u}^{(t-1)}\|^2\big]\tag{from \cref{eq:squared_sum}}\\
		\le\;& \frac{8(K-1)\tilde{L}_g^2\nu^2_t  \alpha^2_{t-1}}{K}\cdot \mathbb{E}\big[\|\tilde{u}^{(t-1)}_i-\bar{u}^{(t-1)}\|^2\big]\nonumber\\
		\label{eq:eq6}
		&+4\alpha^2_{t-1}\nu^2_t\tilde{L}_g^2 \mathbb{E}\big[\|\bar{u}^{(t-1)}\|^2\big]+2(1-\nu_t)^2\sigma^2_g,
	\end{align}
	where the last inequality follows from the fact: ($i$) when $K=1$, $\tilde{u}^{(t-1)}_i=\bar{u}^{(t-1)}$, and when $K\ge 2$, $K-1/K\ge1/2$. Built on Assumptions \ref{asmp:derivative_bounds} and \ref{asmp:bounded_variance}, inequality $(a)$ holds due to \cite[Proposition 5.2]{xu2020improved}:
	\begin{align}
		\|g(\theta_1;\tau) - g(\theta_2;\tau)\|\le L_g\|\theta_1 - \theta_2\|, 
		~\|g(\theta;\tau)\|\le G_g,
		\label{eq:smooth_bounded_g}
	\end{align}
	and inequality $(b)$ follows \cite[Lemma 1]{cortes2010learning} and \cite[Lemma 6.1]{xu2020improved}: for $\tau\sim p(\cdot|\theta)$, we have
	\begin{align}
		\label{eq:importance_mean_var}
		\mathbb{E}\big[w(\theta,\theta';\tau)\big]=1,~\Var(w(\theta,\theta';\tau))\le c_w\|\theta-\theta'\|^2.
	\end{align}
	Plugging \cref{eq:eq6} in \cref{eq:eq5} completes the proof.
\end{proof}

Next, we bound the gradient consensus error in \cref{lem:gradient_consensus_error}.

\begin{lemma}
\label{lem:gradient_consensus_error}
For $t\in(t_q,t_{q+1}]$, the following fact holds:
\begin{align}
    \delta_t&\le \nu^2_t\Big(1+\frac{1}{K} + 8\alpha^2_{t-1}\tilde{L}_g^2K\Big)\delta_{t-1} + 32(1-\nu_t)^2L^2 K^2\nonumber\\
    &\cdot\sum\nolimits^{t}_{{\tau}=t_{q}+1}\alpha^2_{{\tau}}\delta_\tau + 8\alpha^2_{t-1}\tilde{L}_g^2\nu^2_tNK\mathbb{E}\big[\|\bar{u}^{(t-1)}\|^2\big]\nonumber\\
    &+ \frac{8(1-\nu_t)^2\sigma^2_g NK}{D},
    \label{eq:gradient_consensus_error}
\end{align}
where $L$ and $\tilde{L}_g$ are defined in \cref{lem:successive_difference,lem:error_contraction}, respectively.
\end{lemma}
\begin{proof}
We denote $\tilde{d}^{(t-1)}_i\doteq(1/D)\sum^{D}_{j=1}w^{(t)}_{i,j}\cdot g(\theta^{(t-1)}_i;\tau^{(t)}_{i,j})$ and $d^{(t)}_i\doteq(1/D)\sum^{D}_{j=1}g(\theta^{(t)}_i;\tau^{(t)}_{i,j})$. For any ${y}>0$, we have
\begin{align}
    \delta_t\le\;& (1+{y})\nu^2_t\delta_{t-1} + (1+\frac{1}{{y}})\mathbb{E}\bigg[\sum\nolimits^N_{i=1}	\Big\|d^{(t)}_i-\frac{1}{N}\sum^N_{j=1}d^{(t)}_{j}\nonumber\\
    -\;& \nu_t\Big(  \tilde{d}^{(t-1)}_i-\frac{1}{N}\sum\nolimits^N_{j=1}\tilde{d}^{(t-1)}_{j}\Big)\Big\|^2\bigg],
    \label{eq:eq7}
\end{align}
which is derived by substituting the expressions of $\tilde{u}^{(t)}_i$ and $\bar{u}^{(t)}$, extending the norm, and using $2\theta_1^\transpose \theta_2\le q \|\theta_1\|^2 + (1/q)\|\theta_2\|^2$. For the second term of \cref{eq:eq7}, we can write
\begin{align}
    &\mathbb{E}\bigg[\sum^N_{i=1}	\Big\|d^{(t)}_i-\frac{1}{N}\sum^N_{{j}=1}d^{(t)}_{{j}}-\nu_t\Big(  \tilde{d}^{(t-1)}_i-\frac{1}{N}\sum^N_{{j}=1}\tilde{d}^{(t-1)}_{{j}}\Big)\Big\|^2\bigg]\nonumber\\
    &\le 2(1-\nu_t)^2\underbrace{\mathbb{E}\Big[\sum^N_{i=1}\big\|d^{(t)}_i-\frac{1}{N}\sum^N_{{j}=1}d^{(t)}_{{j}}\big\|^2\Big]}_{(a)} \nonumber\\
    &+2\nu^2_t\sum^N_{i=1}\underbrace{\mathbb{E}\Big[\big\|d^{(t)}_i- \tilde{d}^{(t-1)}_i\big\|^2\Big]}_{(b)},
    \label{eq:eq12}
\end{align}
where the inequality is derived from \cref{eq:squared_sum} and the fact: for any $\theta_1,\theta_2,\dots,\theta_n\in\mathbb{R}^d$ and $\bar{\theta}=(1/n)\sum^n_{i=1}\theta_i$, it follows
\begin{align}
    \label{eq:difference_bound}
    \sum^n_{i=1}\|\theta_i - \bar{\theta}\|^2 \le \sum^n_{i=1} \|\theta_i\|^2.
\end{align}
For $(b)$, analogous to \cref{eq:eq8}, we have
\begin{align}
    (b) 
    \label{eq:eq10}
    \le \tilde{L}_g^2\mathbb{E}\big[\|\theta^{(t)}_i-\theta^{(t-1)}_i\|^2\big],
\end{align}
For $(a)$, we have
\begin{align}
    (a) 
    \le\;& \mathbb{E}\Bigg[2\sum\nolimits^N_{i=1}\Big\|d^{(t)}_i - \nabla J(\theta^{(t)}_i) -\Big(\frac{1}{N}\sum\nolimits^N_{{j}=1}d^{(t)}_{{j}}\nonumber\\
    &-\frac{1}{N}\sum\nolimits^N_{{j}=1}\nabla J(\theta^{(t)}_{{j}})\Big) \Big\|^2 + 2\sum\nolimits^N_{i=1}\Big\| \nabla J(\theta^{(t)}_i) \nonumber\\
    &- \frac{1}{N}\sum\nolimits^N_{{j}=1}\nabla J(\theta^{(t)}_{{j}})\Big\|^2\Bigg]\tag{from \cref{eq:squared_sum}}\\
    =\;& \frac{2}{{D}^2}\sum\nolimits^N_{i=1}\underbrace{\mathbb{E}\bigg[\Big\|\sum\nolimits^{D}_{j=1}g(\theta^{(t)}_i;\tau^{(t)}_{i,j}) - \nabla J(\theta^{(t)}_i) \Big\|^2\bigg]}_{(c)} \nonumber\\
    &+ 2\sum\nolimits^N_{i=1}\mathbb{E}\bigg[\underbrace{\Big\| \nabla J(\theta^{(t)}_i) - \frac{1}{N}\sum\nolimits^N_{{j}=1}\nabla J(\theta^{(t)}_{{j}})\Big\|^2}_{\mathclap{(d)}}\bigg]\tag{using \cref{eq:difference_bound} and rearranging terms}\\
    \label{eq:eq11}
    =\;&\frac{2N\sigma^2_g}{D} + 8L^2 \sum\nolimits^N_{i=1}\mathbb{E}\big[\| \theta^{(t)}_i-\bar{\theta}^{(t)}\|^2\big],
\end{align}
Term $(c)$ can be bounded via expanding the norm, eliminating zero expected cross terms, and using \cref{asmp:bounded_variance} as follows:
\begin{align}
    (c) = \sum^{D}_{j=1}\mathbb{E}\bigg[\Big\|g(\theta^{(t)}_i;\tau^{(t)}_{i,j}) - \nabla J(\theta^{(t)}_i) \Big\|^2\bigg]\le D\sigma^2_g.
\end{align}
From \cref{eq:smoothness,eq:squared_sum}, term $(d)$ is bounded by
\begin{align}
    (d) &= \Big\| \nabla J(\theta^{(t)}_i) - \nabla J(\bar{\theta}^{(t)}) +  \nabla J(\bar{\theta}^{(t)})- \sum^N_{{j}=1} \frac{\nabla J(\theta^{(t)}_{{j}})}{N}\Big\|^2\nonumber\\
    &\le 2L^2\big\| \theta^{(t)}_i-\bar{\theta}^{(t)}\big\|^2 +  \frac{2L^2}{N}\sum^N_{{j}=1}\big\| \theta^{(t)}_{{j}}-\bar{\theta}^{(t)}\big\|^2.
\end{align}
Substituting $(a)$, $(b)$ in \cref{eq:eq12} and rearranging terms yield
\begin{align}
    &\mathbb{E}\Bigg[\sum^N_{i=1}	\bigg\|d^{(t)}_i-\frac{1}{N}\sum^N_{j=1}d^{(t)}_{j}-\nu_t\Big(  \tilde{d}^{(t-1)}_i-\frac{1}{N}\sum^N_{j=1}\tilde{d}^{(t-1)}_{j}\Big)\bigg\|^2\Bigg]\nonumber\\
    &\le \frac{4(1-\nu_t)^2\sigma^2_g N}{D} +  2\tilde{L}_g^2\nu^2_t\sum^N_{i=1}\mathbb{E}\big[\|\theta^{(t)}_i-\theta^{(t-1)}_i\|^2\big] \nonumber\\
    &+ 16(1-\nu_t)^2L^2 \sum^N_{i=1}\mathbb{E}\big[\| \theta^{(t)}_i-\bar{\theta}^{(t)}\|^2\big].
    \label{eq:eq9}
\end{align}
Plugging \cref{eq:eq9} into \cref{eq:eq7}, for $t\in(t_q,t_{q+1}]$, we obtain
\begin{align}
    \delta_t
    \le\;& \nu^2_t\Big(1+{y} + 4\alpha^2_{t-1}\tilde{L}_g^2(1+\frac{1}{{y}})\Big)\delta_{t-1}\nonumber\\
    +\;& 16L^2(1-\nu_t)^2 (K-1)(1+\frac{1}{{y}})\sum^{t}_{{\tau}=t_{q}+1}\alpha^2_{{\tau}}\delta_\tau\nonumber\\
    +\;& 4\tilde{L}_g^2\nu^2_t \alpha^2_{t-1} N(1+\frac{1}{{y}})\cdot\mathbb{E}\big[\|\bar{u}^{(t-1)}\|^2\big]\nonumber\\
    +\;& (1+\frac{1}{{y}})\cdot\frac{4(1-\nu_t)^2\sigma^2_g N}{D}, 
\end{align}
where the inequality is derived via using \cref{lem:error_accumulation}, adding and substracting $\bar{u}^{(t-1)}$ in $\|\tilde{u}^{(t-1)}_i\|^2$, and then applying \cref{eq:squared_sum}. Letting ${y}=1/K$ and using $1+K\le2K$ yield the result.
\end{proof}

\cref{lem:error_contraction,lem:gradient_consensus_error} bound the expected gradient error and the expected gradient shift while quantifying the impacts of learning rates, momentum parameters, local steps and interaction frequency on the convergence. We proceed to show how to select these parameters correctly to optimize the communication and interaction complexity.

\begin{lemma}
\label{lem:gradient_error_bound}
For $t\in[T]$, if $\nu_{t+1}/\alpha_t+ 64{\tilde{L}}^2\alpha_{t}/(DN)\le1/\alpha_{t-1}$ and $\nu_{t+1}=1 - c_\nu \alpha^2_{t}$ hold with $\tilde{L}=\max\{\tilde{L}_g,L\}$, we have
\begin{align}
    \alpha_{t}\mathbb{E}\big[\|\bar{\varepsilon}^{(t)}\|^2\big]&\le \frac{DN}{64{\tilde{L}}^2}\bigg(\frac{\mathbb{E}\big[\|\bar{\varepsilon}^{(t)}\|^2\big]}{\alpha_{t-1}}-\frac{\mathbb{E}\big[\|\bar{\varepsilon}^{(t+1)}\|^2\big]}{\alpha_t}\bigg) + \alpha_{t}\delta_t \nonumber\\
    &\cdot  \frac{K-1}{8NK} + \frac{\alpha_{t}}{16}\mathbb{E}\big[\|\bar{u}^{(t)}\|^2\big]+\frac{c^2_\nu\sigma^2_g\alpha^3_{t} }{32{\tilde{L}}^2}.
    \label{eq:eq15}
\end{align}
\end{lemma}
\begin{proof}
From \cref{lem:error_contraction} and $\nu_{t+1}^2\le\nu_{t+1}\le 1$, for all $t\in[T]$, we can write
\begin{align}
    &\frac{\mathbb{E}\big[\|\bar{\varepsilon}^{(t+1)}\|^2\big]}{\alpha_t}-\frac{\mathbb{E}\big[\|\bar{\varepsilon}^{(t)}\|^2\big]}{\alpha_{t-1}}\nonumber\\
    \le\;& \bigg(\frac{\nu_{t+1}}{\alpha_t}-\frac{1}{\alpha_{t-1}}\bigg)\mathbb{E}\big[\|\bar{\varepsilon}^{(t)}\|^2\big] +\frac{4\tilde{L}_g^2\alpha_{t}}{DN}\mathbb{E}\big[\|\bar{u}^{(t)}\|^2\big]\nonumber\\
    +\;&\frac{8(K-1)\tilde{L}_g^2\alpha_{t}\delta_t}{KD N^2}+\frac{2\sigma^2_g(1-\nu_{t+1})^2}{DN\alpha_t}.
    \label{eq:eq14}
\end{align}
Using the contions and rearranging terms yields the result.
\end{proof}

\begin{lemma}
\label{lem:accumulated_grad_consensus_error}
For each $t\in(t_q,t_{q+1}]$, if $c_\nu\le128\sqrt{2}L^2/(DN)$, $\alpha_{t}\le1/(16{\tilde{L}} K)$, and $\nu_{t}=1 - c_\nu \alpha^2_{t-1}$ hold, then we have
\begin{align}
    \frac{K-1}{4NK}\sum^{t}_{{\tau}=t_q+1}\alpha_{{\tau}}\delta_\tau\le\ \sum^{t-1}_{{\tau}=t_{q}}\frac{\alpha_{{\tau}}}{64} \mathbb{E}\big[\|\bar{u}^{({\tau})}\|^2\big]+\frac{ c_\nu^2\sigma^2_g \alpha^3_{{\tau}}}{64D{\tilde{L}}^2}.
    \label{eq:eq32}
\end{align}
\end{lemma}
\begin{proof}
Due to $\nu^2_t\le1$,  $\alpha_t\le 1/(16K\tilde{L}_g)$ and \cref{lem:gradient_consensus_error}, for each $t\in(t_q,t_{q+1}]$, we have
\begin{align}
    \delta_t&\le \left(1+\frac{33}{32K}\right)\vphantom{\sum^N_{i=1}}\delta_{t-1} + 32K^2L^2(1-\nu_t)^2 \sum^{t}_{{\tau}=t_{q}+1}\alpha^2_{{\tau}}\delta_\tau \nonumber\\
    &+ \frac{ N\tilde{L}_g \alpha_{t-1}}{2} \mathbb{E}\big[\|\bar{u}^{(t-1)}\|^2\big]+ \frac{8 NK\sigma^2_g(1-\nu_t)^2}{D}.
    \label{eq:eq13}
\end{align}
Due to $\delta_{t_q}=0$, applying \cref{eq:eq13} recursively for ${\tau}\in(t_q,t]$, we obtain
\begin{align}
    \delta_t
    \le\;&96K^2L^2c_\nu^2\sum^{t-1}_{{\tau}=t_q}  \alpha^4_{{\tau}} \sum^{{\tau}+1}_{{n}=t_{q}+1}\alpha^2_{{n}}\delta_n +\frac{24 c_\nu^2\sigma^2_g NK}{D}\sum^{t-1}_{{\tau}=t_q} \alpha^4_{{\tau}}\nonumber\\
    &+\frac{ 3N\tilde{L}_g}{2}\sum^{t-1}_{{\tau}=t_q}\alpha_{{\tau}} \mathbb{E}\big[\|\bar{u}^{({\tau})}\|^2\big] \tag{from $t-{\tau}-1\le K$ and $(1+33/(32K))^K\le \mathrm{e}^{33/22}\le3$}\\
    \le\;&96K^3L^2c_\nu^2\bigg(\frac{1}{16LK}\bigg)^5 \sum^{t-1}_{{\tau}=t_{q}+1}\alpha_{{\tau}}\delta_\tau +\frac{3 c_\nu^2\sigma^2_g N}{2B{\tilde{L}}}\sum^{t-1}_{{\tau}=t_q} \alpha^3_{{\tau}}\nonumber\\
    &+\frac{ 3N\tilde{L}_g}{2}\sum^{t-1}_{{\tau}=t_q}\alpha_{{\tau}} \mathbb{E}\big[\|\bar{u}^{({\tau})}\|^2\big],
    \label{eq:eq22}
\end{align}
where the last inequality holds from $\alpha_{t}\le1/(16{\tilde{L}} K)$. Multiplying \cref{eq:eq22} by $\alpha_{t}$ and summing over $t=t_q + 1$ to ${\tau}\in(t_q,t_{q+1}]$, we get
\begin{align}
    \sum^{{\tau}}_{t=t_q+1}\alpha_t\delta_t
    &\le\frac{ 3N}{32} \sum^{{\tau}-1}_{{n}=t_{q}}\alpha_{{n}} \mathbb{E}\big[\|\bar{u}^{({n})}\|^2\big]+\frac{3 c_\nu^2\sigma^2_g N}{32D{\tilde{L}}^2} \sum^{{\tau}-1}_{{n}=t_{q}} \alpha^3_{{n}}\nonumber\\
    &+96K^4L^2c_\nu^2\bigg(\frac{1}{16LK}\bigg)^6 \sum^{{\tau}}_{{n}=t_{q}+1}\alpha_{{n}}\delta_n,
\end{align}
where we use ${\tau}-t_q\le K$, $\alpha_{t}\le1/(16{\tilde{L}} K)$ and $\tilde{u}^{(t_q)}_i=\bar{u}^{(t_q)}$. Rearranging terms yields
\begin{align}
    &\bigg(1-96K^4L^2c_\nu^2\Big(\frac{1}{16LK}\Big)^6\bigg)\sum^{{\tau}}_{t=t_q+1}\alpha_t\delta_t\nonumber\\
    &\le\frac{ 3N}{32} \sum^{{\tau}-1}_{{n}=t_{q}}\alpha_{{n}} \mathbb{E}\big[\|\bar{u}^{({n})}\|^2\big]+\frac{3 Nc_\nu^2\sigma^2_g}{32D{\tilde{L}}^2} \sum^{{\tau}-1}_{{n}=t_{q}} \alpha^3_{{n}}.
\end{align}
Due to $c_\nu\le128\sqrt{2}L^2/(DN)$, $1-96K^4L^2c_\nu^2 1/(16LK)^6\ge1/4$ holds, thereby completing the proof.
\end{proof}

Now, we are ready to establish the convergence property.

\begin{theorem}
\label{thm:gradient_bound}
Suppose that the sequences of learning rates and momentum parameters across interaction steps are selected as
$\alpha_t = c_\alpha/(c_t+\sigma^2_g t)^{1/3}$ and $\nu_{t+1}=1 - c_\nu \alpha^2_{t}$ respectively. Then, the following fact holds true:
\begin{align}
    &\frac{1}{T}\sum\nolimits^{T}_{t=1} \mathbb{E}\Big[\big\|\nabla J(\bar{\theta}^{(t)}) \big\|^2\Big]\nonumber\\
    \le\;&\bigg(\frac{32\tilde{L}K}{T} + \frac{2\tilde{L}}{(DN)^{2/3}T^{2/3}}\bigg) \bigg(J(\bar{\theta}^{(1)}) -J^*\bigg)\nonumber\\
    + \;&\bigg(\frac{2^{15} K}{T} + \frac{2^{11}}{(DN)^{2/3}T^{2/3}}\bigg)\bigg(1+\ln\left(1+T\right)\bigg)\sigma^2_g\nonumber\\
    +\;& \bigg(\frac{8DK^2}{\tilde{D}T} + \frac{DK}{2\tilde{D}(DN)^{2/3}T^{2/3}}\bigg)\sigma^2_g,
    \label{eq:eq31}
\end{align}
where $J^*\doteq\min_\theta J(\theta)$, and $c_\alpha$, $c_\nu$ and $c_t$ are defined as
\begin{align*}
    c_t\doteq\max\left\{\frac{c^3_\nu c^3_\alpha}{2^{12}K^3{\tilde{L}}^3},2^{12}K^3D^2N^2\sigma^2_g - \sigma^2_g t,2\sigma^2_g\right\}\\
    c_\nu\doteq\frac{{\tilde{L}}^2}{24 K(DN)^2}+\frac{64{\tilde{L}}^2}{DN},~c_\alpha\doteq\frac{(DN\sigma_g)^{2/3}}{{\tilde{L}}} .
\end{align*} 
\end{theorem}
\begin{proof}
First, it can be easily verified that $\alpha_t\le1/(16{\tilde{L}} K)$. Due to $c_t\le c_{t-1}$, the following holds:
\begin{align}
    \frac{1}{\alpha_t}-\frac{1}{\alpha_{t-1}}
    &\le\frac{(c_t+\sigma^2_g t)^{1/3}}{c_\alpha}-\frac{(c_{t}+\sigma^2_g (t-1))^{1/3}}{c_\alpha}\\
    &\le\frac{\sigma^2_g}{3c_\alpha(c_{t}+\sigma^2_g (t-1))^{2/3}}\tag{from the concavity of $x^{1/3}$: $(x+y)^{1/3} - x^{1/3}\le\frac{y}{3x^{2/3}}$}\\
    &\le\frac{2^{2/3}\sigma^2_g }{3c^3_\alpha}\cdot \frac{c^2_\alpha}{(c_{t}+\sigma^2_g t)^{2/3}}\tag{from  $c_t\ge 2\sigma^2_g$}\\
    &\le\bigg(c_\nu -\frac{64{\tilde{L}}^2}{DN}\bigg)\alpha_{t},
    \label{eq:eq35}
\end{align}
where we use $\alpha_t\le1/(16\tilde{L}_gK)$ and the definitions of $\alpha_t$ and $c_\nu$.
Based on \cref{eq:eq35}, substituting \cref{eq:eq15,eq:eq32} in \cref{eq:successive_difference_2}, using $\alpha_t\le1/(16KL)$ and $D\ge1$ and summing over $t=t_q+1$ to $t_{q+1}$, we obtain
\begin{align}
    &\mathbb{E}\big[J(\bar{\theta}^{(t_{q+1}+1)})\big] + \frac{DN\mathbb{E}\big[\|\bar{\varepsilon}^{(t_{q+1}+1)}\|^2\big]}{64{\tilde{L}}^2\alpha_{t_{q+1}}}  \nonumber\\
    \le\;& \mathbb{E}\big[J(\bar{\theta}^{(t_q+1)})\big] + \frac{DN\mathbb{E}\big[\|\bar{\varepsilon}^{(t_q+1)}\|^2\big]}{64{\tilde{L}}\alpha_{t_q}}+\frac{\alpha_{t_q}}{64} \mathbb{E}\big[\|\bar{u}^{(t_q)}\|^2\big] \nonumber\\
    -\;&\sum^{t_{q+1}}_{t=t_q+1} \bigg(\frac{27\alpha_{t}}{64}-\frac{\alpha^2_t L}{2}\bigg)\mathbb{E}\big[\|\bar{u}^{(t)}\|^2\big] + \frac{3c^2_\nu\sigma^2_g}{64{\tilde{L}}^2}\sum^{t_{q+1}}_{t=t_q+1}\alpha^3_{t}\nonumber\\
    +\;&\frac{ c_\nu^2\sigma^2_g}{64D{\tilde{L}}^2}  \alpha^3_{t_q} - \sum^{t_{q+1}}_{t=t_q+1}\frac{\alpha_t}{2}\mathbb{E}\big[\|\nabla J(\bar{\theta}^{(t)}) \|^2\big]
\end{align}
Based upon $\bar{u}^{(0)}=0$, $D\ge1$ and $\alpha_t\le1/(16LK)$, suming over $q\in\{0,1,\dots,M-1\}$ yields
\begin{align}
    \mathbb{E}\big[J(\bar{\theta}^{(t_{M}+1)})\big]   &\le J(\bar{\theta}^{(1)})+\frac{DN\mathbb{E}\big[\|\bar{\varepsilon}^{(1)}\|^2\big]}{64{\tilde{L}}^2\alpha_{0}} +\frac{c^2_\nu\sigma^2_g }{16{\tilde{L}}^2}\sum^{t_{M}}_{t=0}\alpha^3_{t}\nonumber\\
    \label{eq:eq33}
    &-\sum^{t_{M}}_{t=1} \frac{\alpha_{t}}{2}\mathbb{E}\big[\|\nabla J(\bar{\theta}^{(t)}) \|^2\big]. 
\end{align}
Note that $T=MK$. Rearranging terms in \cref{eq:eq33}, we have
\begin{align}
    \frac{1}{T}\sum^{T}_{t=1} \mathbb{E}\big[\|\nabla J(\bar{\theta}^{(t)}) \|^2\big]
    &\le   \frac{2(J(\bar{\theta}^{(1)}) - J^*)}{\alpha_{T}T} + \frac{D\sigma^2_g}{32{\tilde{L}}^2\tilde{D}\alpha_{0}\alpha_{T}T} \nonumber\\
    &+\frac{c^2_\nu\sigma^2_g }{8{\tilde{L}}^2 \alpha_{T}T}\sum^{T}_{t=0}\alpha^3_{t},
    \label{eq:eq24}
\end{align}
where we use $\mathbb{E}[\|\bar{\varepsilon}^{(1)}\|^2] \le  \sigma^2_g/(\tilde{D}{N})$ (akin to \cref{eq:eq5}). For term $\sum^{T}_{t=0}\alpha^3_{t}$, due to $c_t\ge2\sigma^2_g\ge\sigma^2_g$, we obtain
\begin{align}
    \sum^{T}_{t=0}\alpha^3_{t} 
    =\frac{c^3_\alpha}{\sigma^2_g} \sum^{T}_{t=0}\frac{1}{1+ t}\le\frac{c^3_\alpha}{\sigma^2_g}\big(1+\ln\left(1+T\right)\big),
    \label{eq:eq25}
\end{align}
based on the relationship:
\begin{align}
    \sum^T_{t=1}\frac{x_t}{x_0+\sum^t_{{\tau}=1}x_{{\tau}}}\le\ln\bigg(1+\frac{\sum^t_{{\tau}=1}x_{{\tau}}}{x_0}\bigg),
\end{align}
 with $x_0=1$ and $x_1,x_2,\dots,x_T=1$. From the definition of $c_\alpha$ and $c_t$ and the fact, $c_\nu\le2^7{\tilde{L}}^2/(DN)$, it is clear that
 \begin{align}
    c_T
    &\le\sigma^2_g\max\bigg\{\frac{2^9}{DNK^3},2^{12}K^3(DN)^2 - T,2\bigg\}.
    \label{eq:eq26}
 \end{align}
 Accordingly, we have $c_T\le2^{12}K^3(DN)^2\sigma^2_g$.
 Drawing on the definition of $\alpha_t$ and $c_\alpha$, we can bound term $1/(\alpha_T T)$ by
 \begin{align}
    \frac{1}{\alpha_T T} 
    &\le \frac{c^{1/3}_T}{c_\alpha T} + \frac{\sigma^{2/3}_g}{c_\alpha T^{2/3}}\tag{from $(x+y)^{1/3}\le x^{1/3}+y^{1/3}$}\\
    &\le \frac{16K{\tilde{L}}}{T} + \frac{{\tilde{L}}}{(DN)^{2/3}T^{2/3}}.
    \label{eq:eq27}
 \end{align}
 Regarding the second term in \cref{eq:eq24}, we can write
 \begin{align}
    \frac{D\sigma^2_g}{32{\tilde{L}}^2\tilde{D}\alpha_{0}\alpha_T T} &\le \bigg( \frac{16{\tilde{L}}K}{T} + \frac{{\tilde{L}}}{(DN)^{2/3}T^{2/3}}\bigg)	\frac{\sigma^2_gc^{1/3}_0 D}{32{\tilde{L}}^2 \tilde{D}c_\alpha}\tag{from \cref{eq:eq27} and the definition of $\alpha_{0}$}\\
    &\le \bigg( \frac{16{\tilde{L}}K}{T} + \frac{{\tilde{L}}}{(DN)^{2/3}T^{2/3}}\bigg)	\frac{16\sigma^2_gK\tilde{L}D}{32{\tilde{L}}^2\tilde{D}}\tag{from $c_0\le 2^{12}\tilde{L}^3K^3c^3_\alpha$}\\
    &\le\frac{8DK^2\sigma^2_g}{\tilde{D}T} + \frac{DK\sigma^2_g}{2\tilde{D}(DN)^{2/3}T^{2/3}}.
    \label{eq:eq29}
 \end{align}
 Regarding the third term in \cref{eq:eq24}, we have
 \begin{align}
    \frac{c^2_\nu c^3_\alpha}{8\tilde{L}^2\alpha_T T}&\le\bigg( \frac{16\tilde{L}K}{T} + \frac{\tilde{L}}{(DN)^{2/3}T^{2/3}}\bigg)\bigg(\frac{2^7\tilde{L}^2}{DN}\bigg)^2\frac{(DN)^2\sigma^2_g}{8\tilde{L}^5}\tag{from $c_\nu\le\frac{2^7\tilde{L}^2}{DN}$ and $c_\alpha=(DN)^{2/3}\sigma^{2/3}_g/\tilde{L}$}\\
    &\le\frac{2^{15} K\sigma^2_g}{T} + \frac{2^{11}\sigma^2_g}{(DN)^{2/3}T^{2/3}}.
    \label{eq:eq30}
 \end{align}
 Finally, substituting the bounds in \cref{eq:eq25,eq:eq27,eq:eq29,eq:eq30} in \cref{eq:eq24}, we complete the proof.
\end{proof}




\section{Experiment}
In this section, we empirically evaluate the proposed method and answer the following key questions:
\begin{itemize}
    \item [1)] How does \alg{} perform on standard benchmarks in comparison to the existing FRL algorithms?
    \item [2)] What is the communication/interaction cost of \alg{}?
    \item [3)] How does \alg{} speed up with the number of agents?
    \item [4)] What is the impact of interaction frequency?
\end{itemize}

\subsection{Experimental Setup}

\subsubsection{Environments}
The experiments are carried out on six challenging \texttt{gym} environments~\cite{openai2016gym}, including Classic Control tasks (i.e., Cartpole and Pendulum), continuous control MuJoCo tasks (i.e., Halfcheetah and Hopper), and image-based Atari games (i.e., Pong and Breakout), as shown in \cref{fig:environment}.

\begin{figure}[h]
    \vspace{-0.5em}
    \centering  \includegraphics[width=0.995\linewidth]{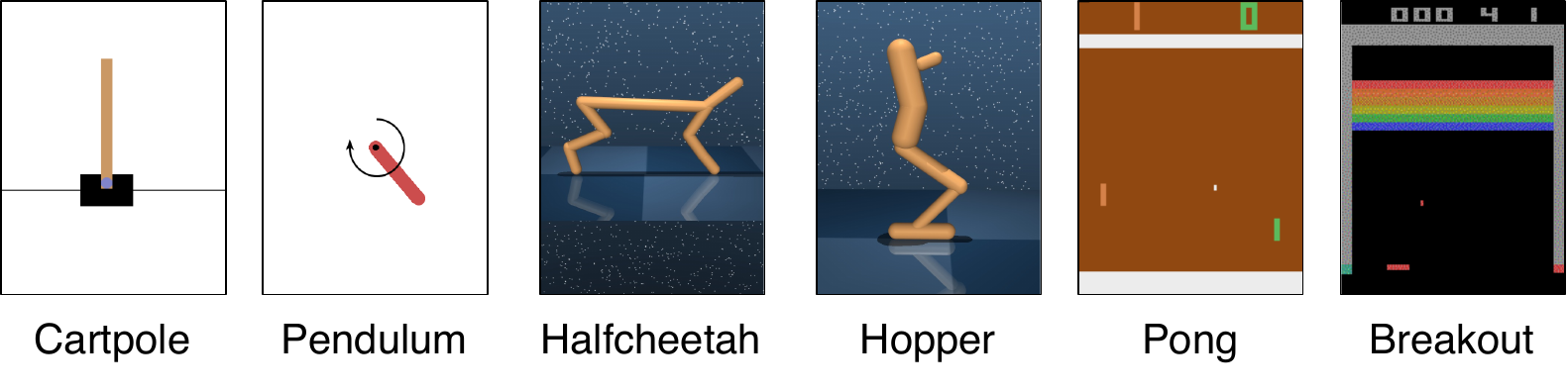}
    \vspace{-2em}
    \caption{Benchmark environments.}
    \vspace{-0.7em}
    \label{fig:environment}
\end{figure}

\subsubsection{Baselines}
We compare our proposed algorithm with the following three baseline methods:
\begin{itemize}
\item  \emph{Federated Policy Gradient with the
Byzantine Resilience} (\texttt{FedPG-BR})~\cite{fan2021fault}, a recent fault-tolerant FRL algorithm;

\item  \emph{Federated Double Q-learning} (\texttt{Fed-DQN})~\cite{khodadadian2022federated}, an FRL algorithm that combines \texttt{FedAvg} with \texttt{DQN};

\item  \emph{Federated Soft Actor-Critic} (\texttt{Fed-SAC}), an FRL algorithm combining \texttt{FedAvg} with \texttt{SAC}~\cite{haarnoja2018soft}.
\end{itemize}

\subsubsection{Implementation}
The policy is represented as a 2-layer feedforward neural network with 16 hidden units for Classic Control and 256 for the other tasks. It uses ReLU activation functions and Tanh Gaussian outputs. Guided by the analytical results, we set the momentum parameter as $\nu^{(t)}=1-3\alpha^{(t)}$ and the stepsize as $\alpha^{(t)}=10^{-4}\times 0.99^{-t}$, they are both decrease with updating step $t$. The discount factor is set to $\gamma=0.99$. In each round, the number of local updating steps is set as $K=10$ and $K=20$ for Classic Control domains and other domains, respectively. We sample $D=20$ trajectories in each updating step. In addition, we implement the code using Pytorch 1.8.1 framework and run the experiments on Ubuntu 18.04.2 LTS with $8$ NVIDIA GeForce RTX A6000 GPUs.

\subsection{Experimental Result}

\begin{figure*}[tb]
    \centering
    \vspace{-0.3em}
    \subfigure[Results on Cartpole.]{\includegraphics[width=0.49\columnwidth]{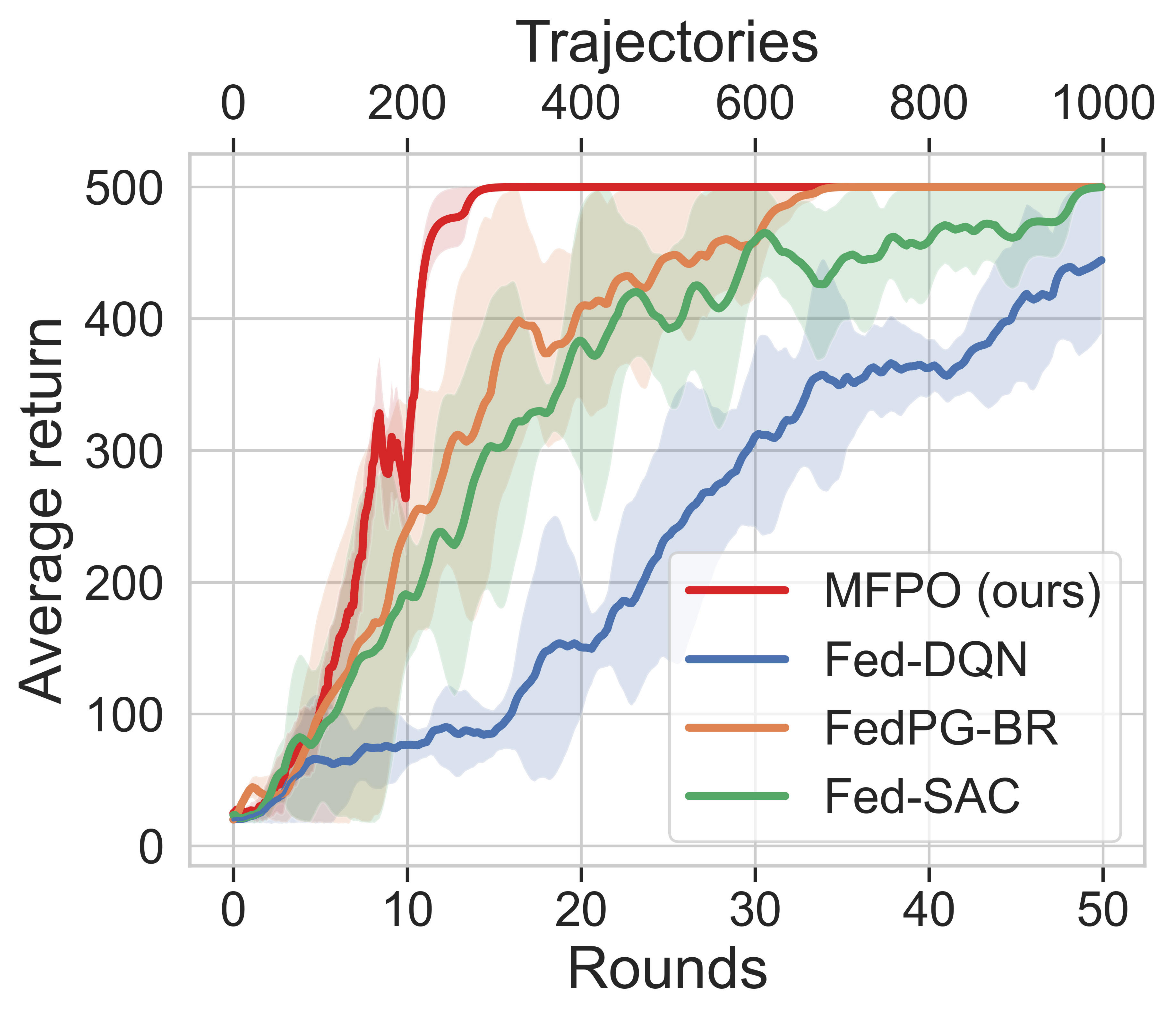}}
    \subfigure[Results on Pendulum.]{\includegraphics[width=0.49\columnwidth]{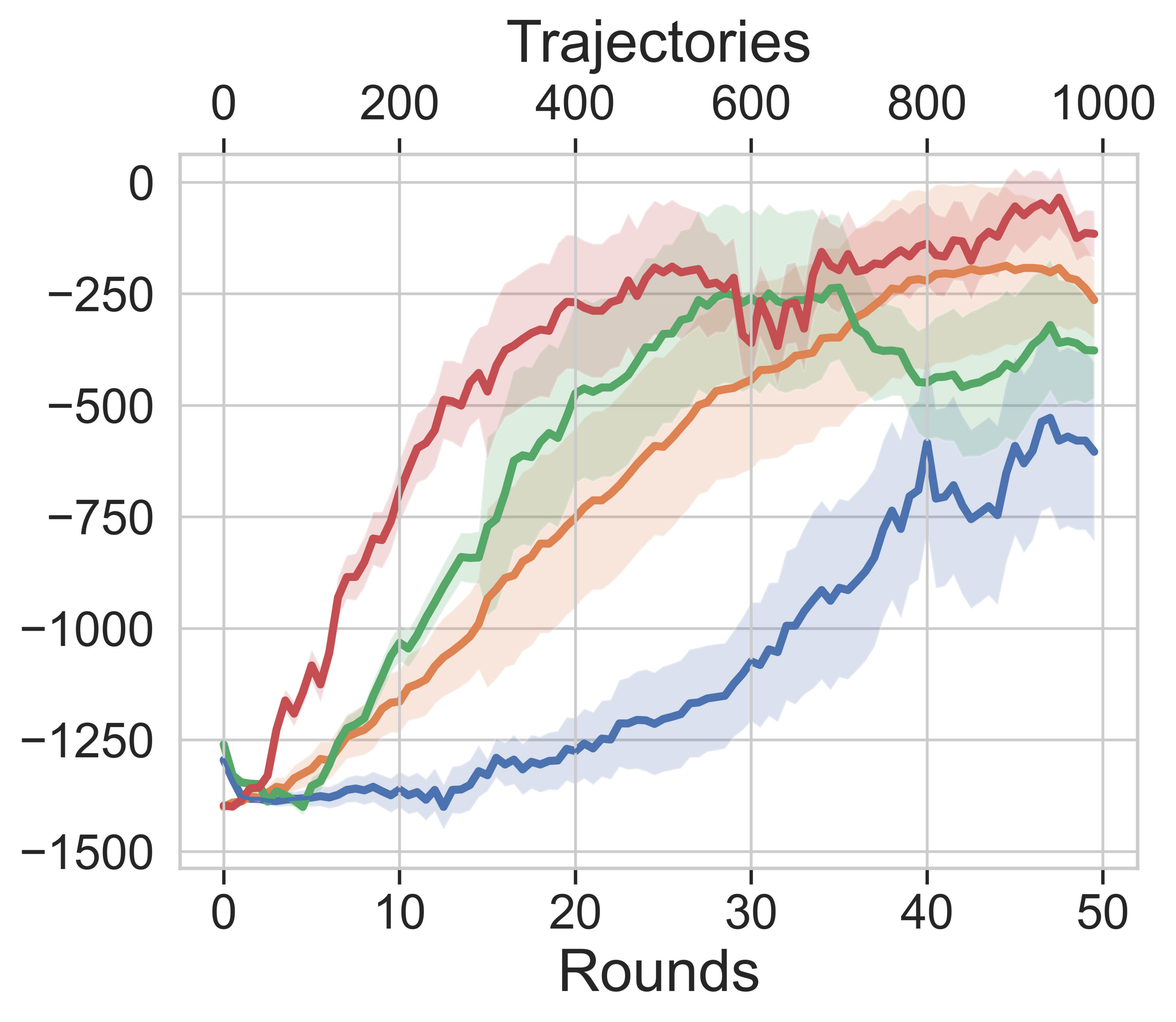}}
    \subfigure[Results on Halfcheetah.]{\includegraphics[width=0.49\columnwidth]{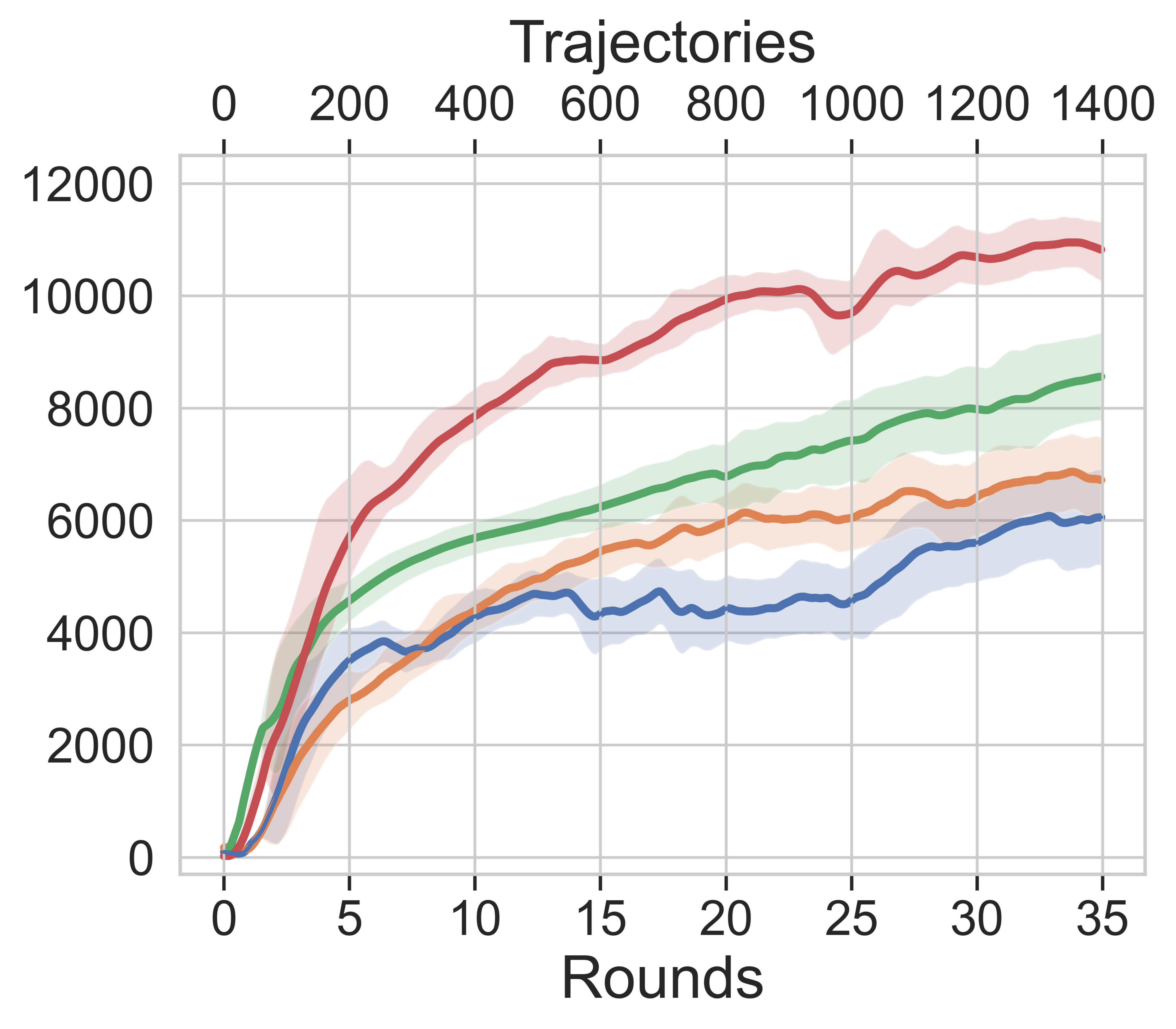}}
    \subfigure[Results on Hopper.]{\includegraphics[width=0.49\columnwidth]{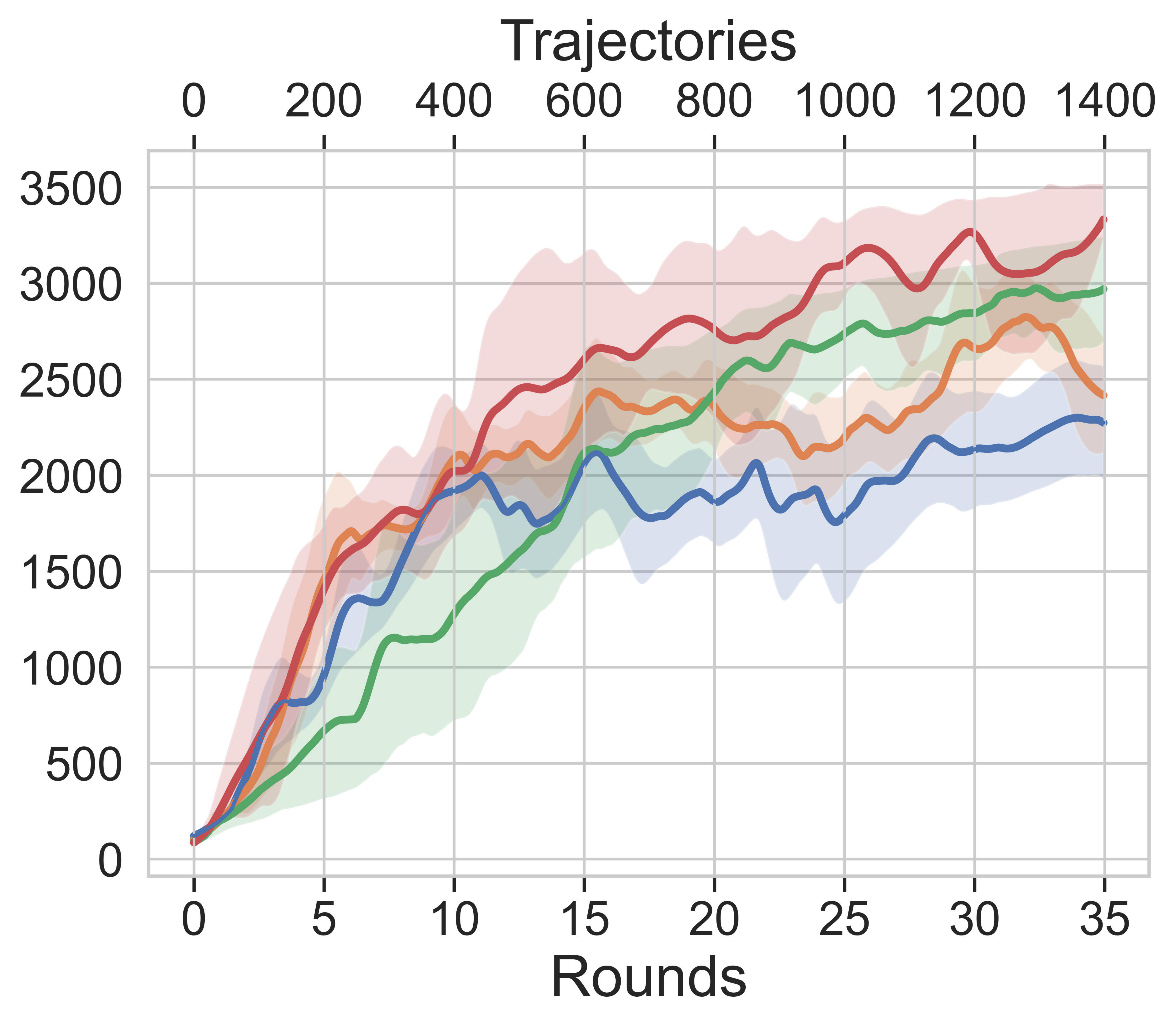}}
    \vspace{-0.5em}
    \caption{Convergence results on Classic Control and Mujoco domains.}
    \label{fig:performance_control}
\end{figure*}

\begin{figure*}[!t]
  \centering
  \vspace{-0.5cm}
\begin{minipage}[c]{0.48\textwidth}
\subfigure[Results on Pong.]{\includegraphics[width=0.49\columnwidth]{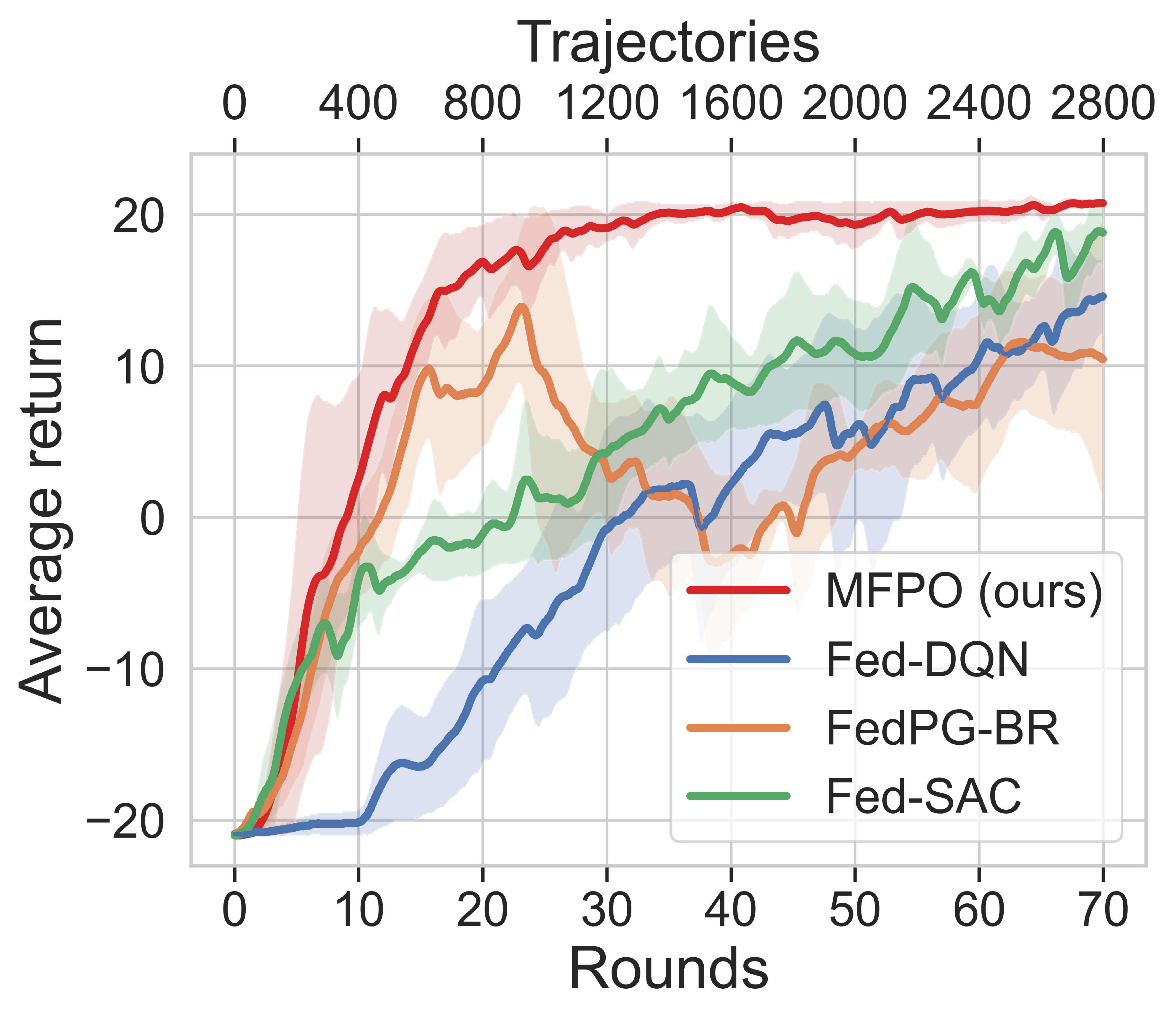}}
    \subfigure[Results on Breakout.]{\includegraphics[width=0.49\columnwidth]{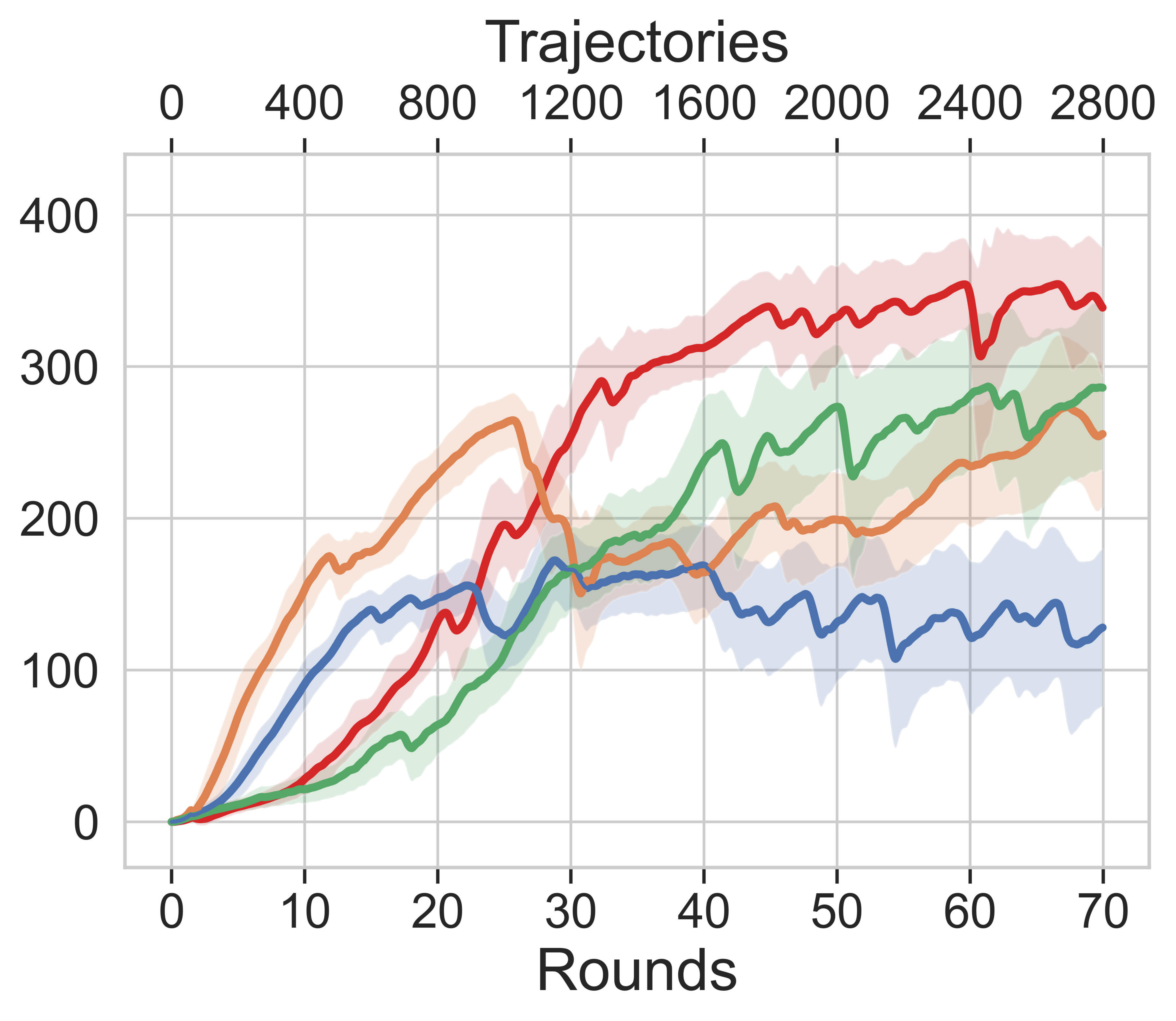}}
    \vspace{-0.5em}
    \caption{Convergence results on Atari games.}
    \label{fig:performance_atari}
  \end{minipage}
  \hspace{0.05cm} 
  \begin{minipage}[c]{0.48\textwidth}
  \subfigure[Results on Cartpole.]{\includegraphics[width=0.49\columnwidth]{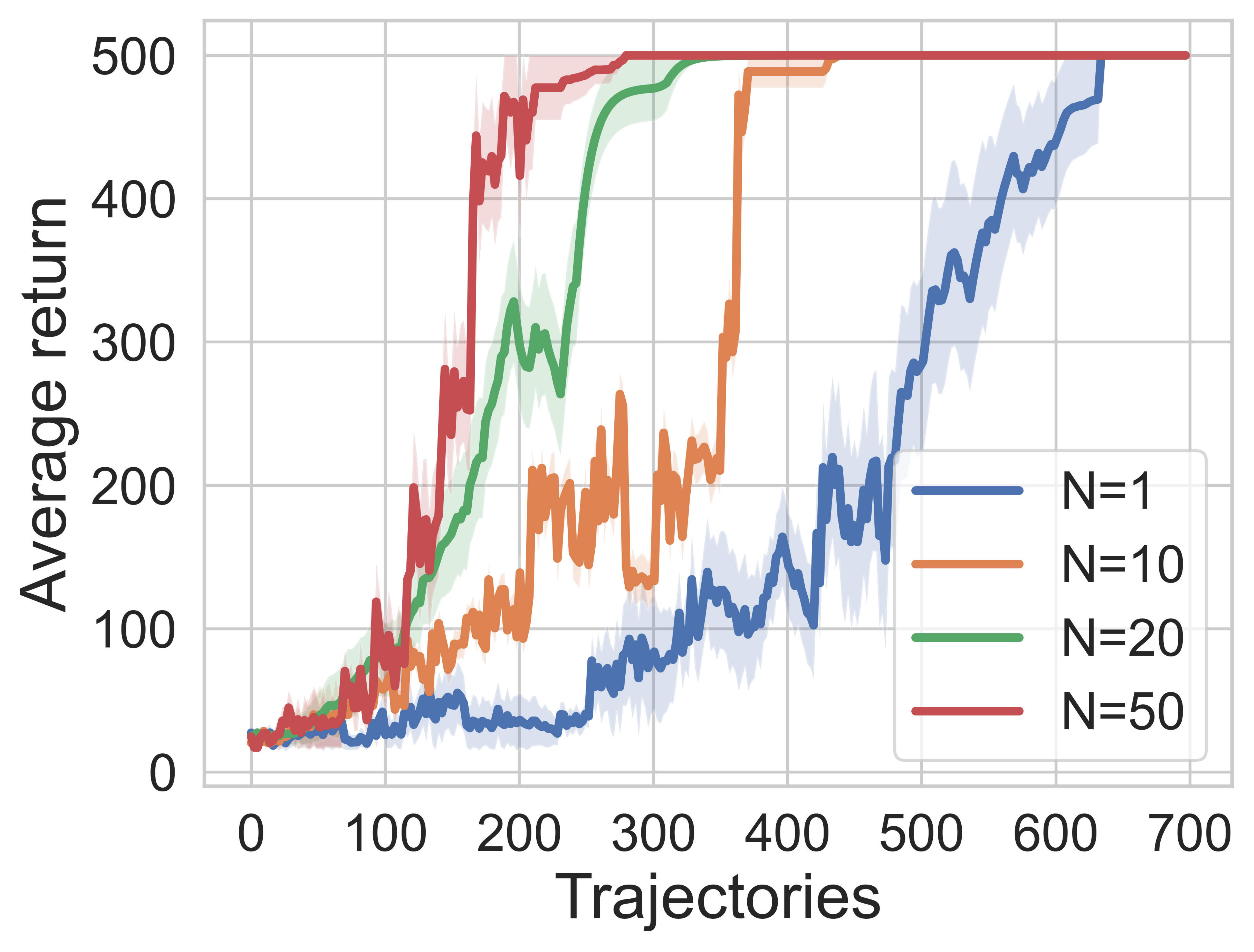}}
    \subfigure[Results on Halfcheetah.]{\includegraphics[width=0.49\columnwidth]{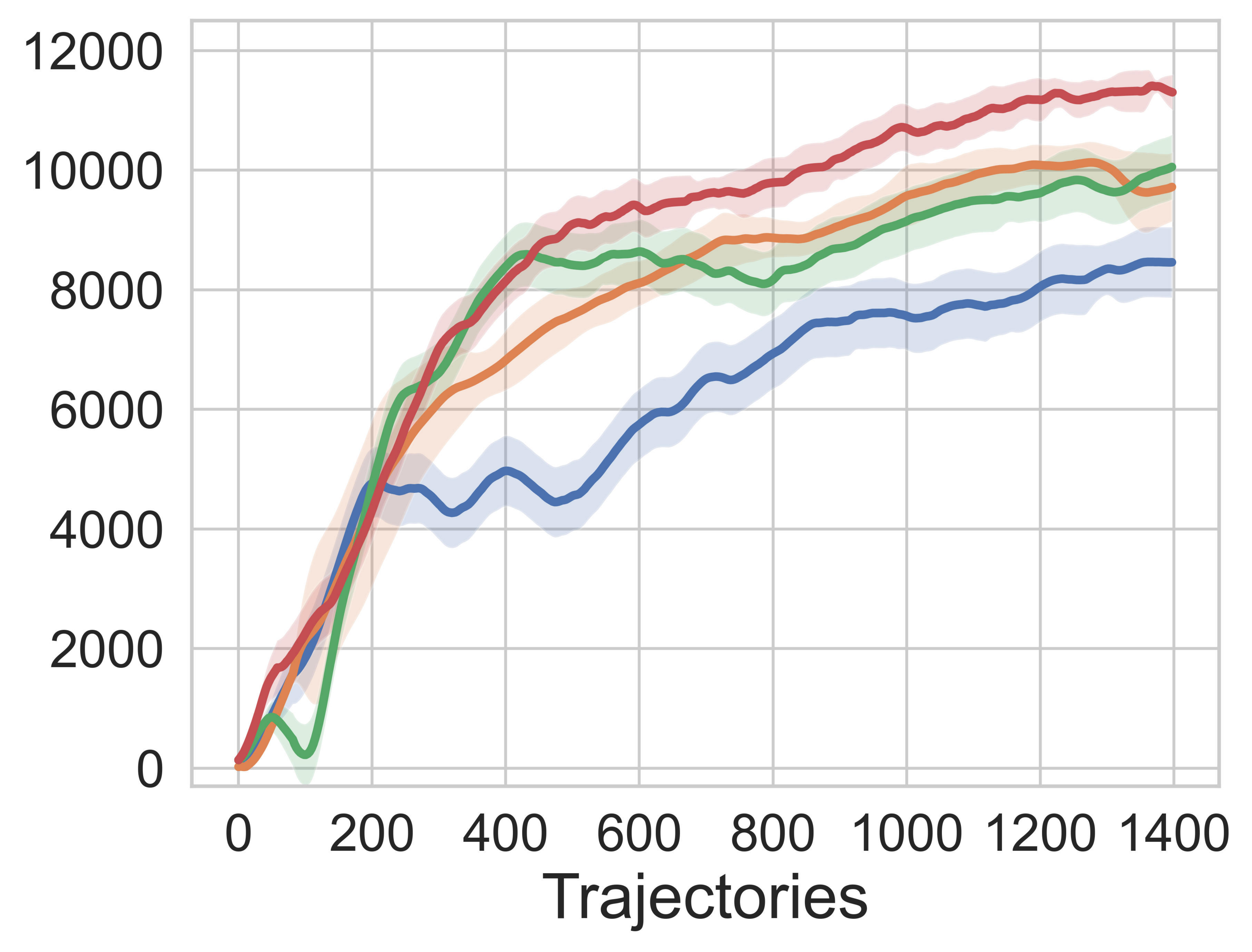}}
    \vspace{-0.5em}
    \caption{Performance uner different number of agents.}
    \label{fig:agent_number}
  \end{minipage} 
  \end{figure*}

  \begin{figure*}[!t]
  \centering
 \vspace{-0.5cm}
  \begin{minipage}[c]{0.48\textwidth}
  \subfigure[Results on Cartpole.]{\includegraphics[width=0.49\columnwidth]{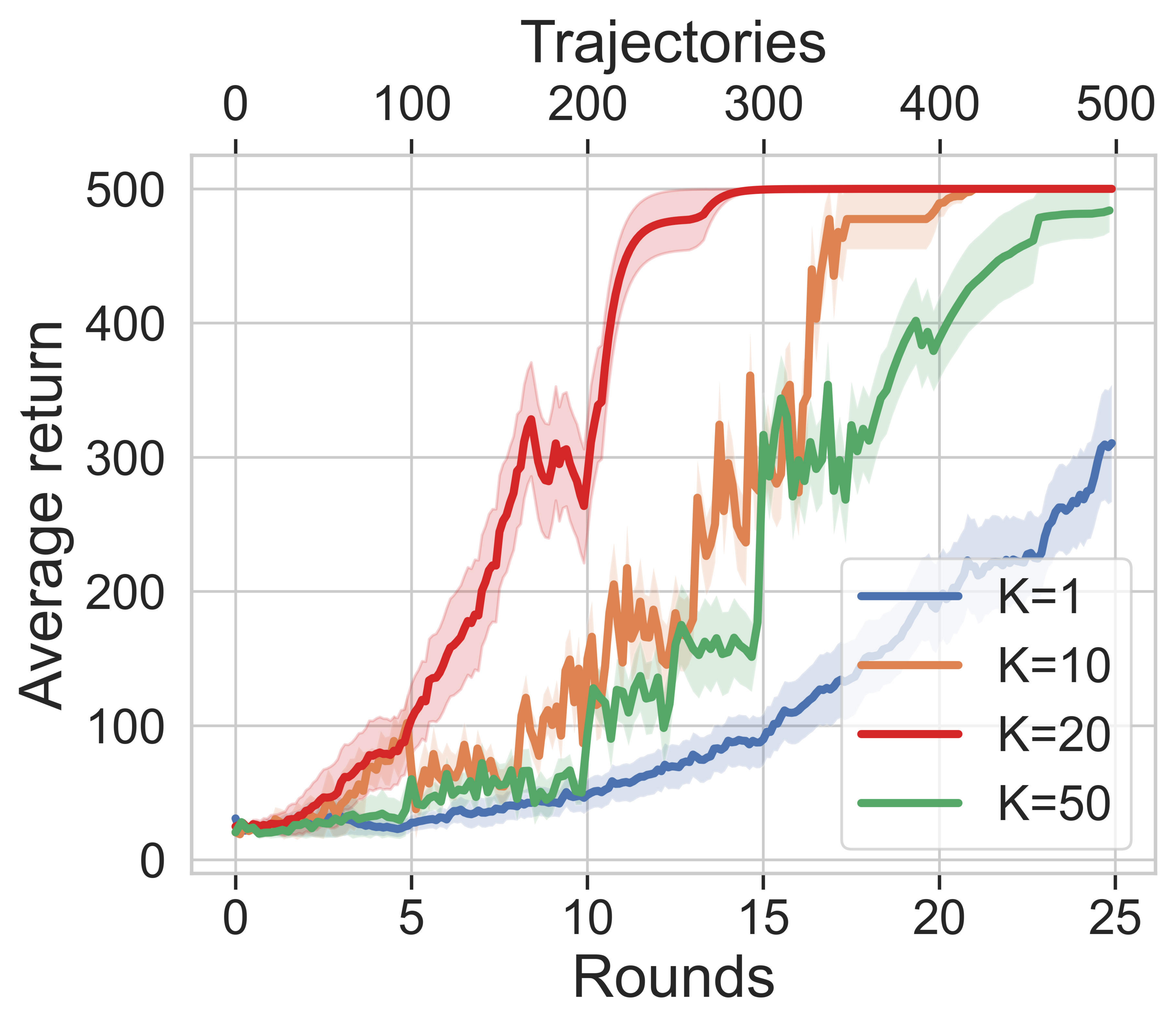}}
    \subfigure[Results on Halfcheetah.]{\includegraphics[width=0.49\columnwidth]{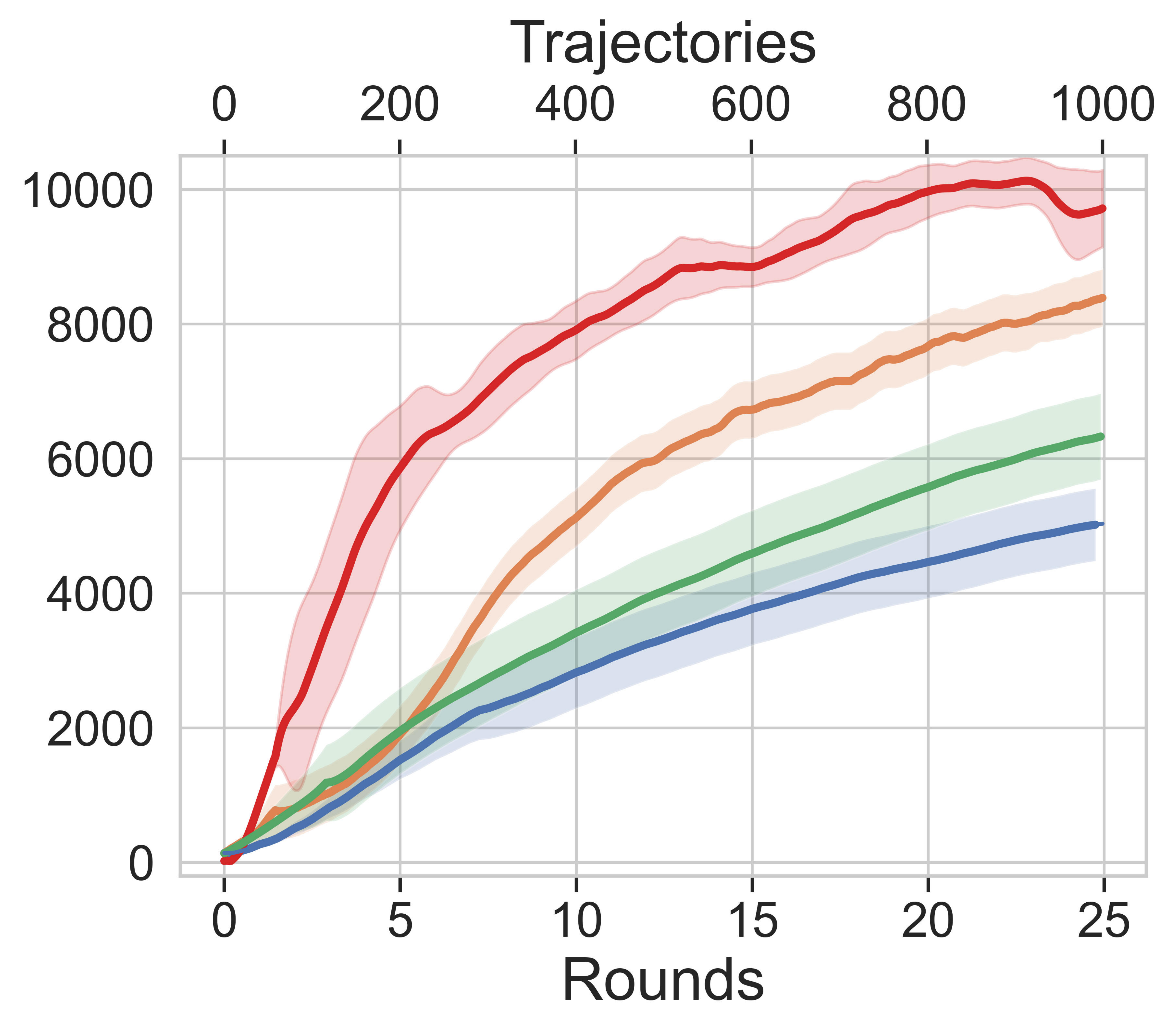}}
    \vspace{-0.5em}
    \caption{Impact of local steps on performance.}
    \label{fig:local_steps}
  \end{minipage} 
  \hspace{0.05cm} 
  \begin{minipage}[c]{0.48\textwidth}
\subfigure[Results on Cartpole.]{\includegraphics[width=0.49\columnwidth]{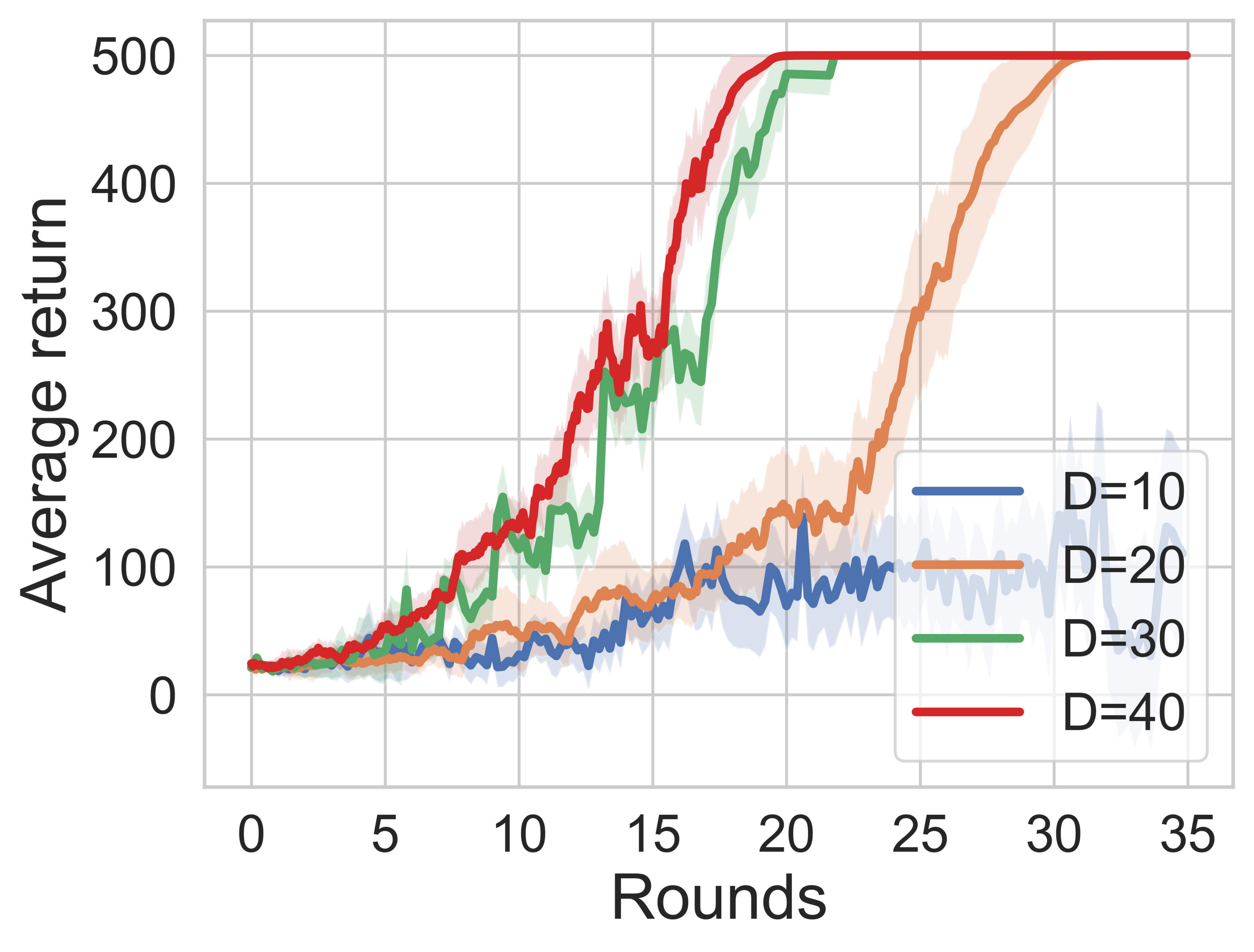}}
    \subfigure[Results on Halfcheetah.]{\includegraphics[width=0.49\columnwidth]{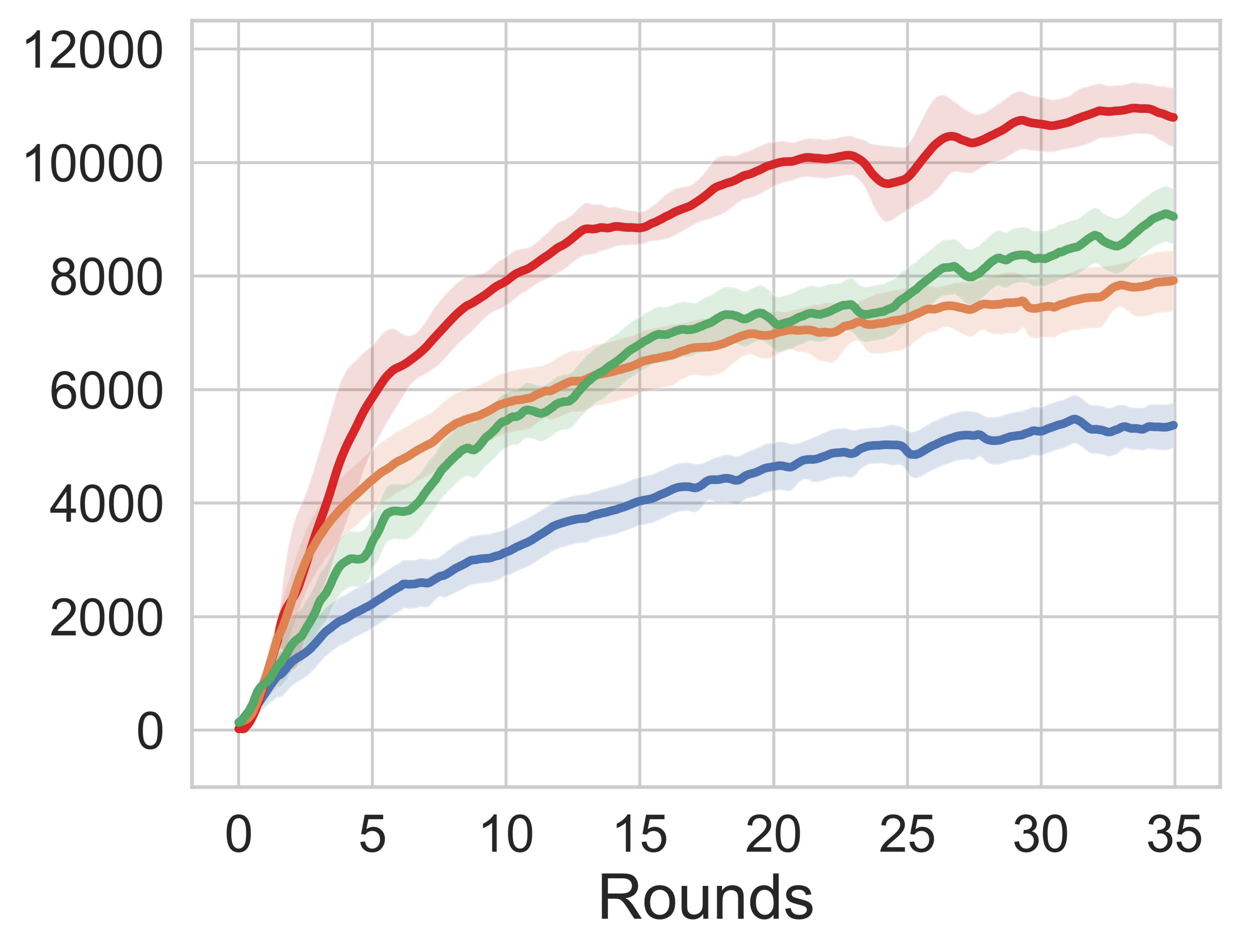}}
    \vspace{-0.5em}
    \caption{Performance under different batch sizes.}
    \label{fig:batch_sizes}
  \end{minipage}
    \vspace{-0.5cm}
  \end{figure*}

\subsubsection{Comparative results}
To answer the first and second questions arised above, we provide comparison results of proposed \alg{} with three baselines. 
As shown in \cref{tab:performance}, \alg{} yields the best performance with a wide margin among all tasks. \cref{fig:performance_control,fig:performance_atari} show that \alg{} achieves higher returns while incurring relatively low communication and interaction costs (often by less than 30 rounds and 1000 trajectories), especially in complex and high-dimensional environments. In contrast to the fluctuating performance observed in the baselines, \alg{} holds remarkable stability. This demonstrates the efficacy of the momentum-assisted mechanism introduced in controlling policy gradient variates.

\begin{table}[htpb]
\centering
\vspace{-0.25em}
\renewcommand\arraystretch{1.25}
\caption{Average scores on different environments.}
\vspace{-0.5em}
\resizebox{\columnwidth}{!}{
\begin{tabular}{@{}lrrrr@{}}
\toprule
~~\textbf{Environment}       & \textbf{\texttt{Fed-DQN}} & \textbf{\texttt{FedPG-BR}}     & \textbf{\texttt{Fed-SAC}}      & \textbf{\alg{} (ours)}~~    \\ \midrule
~~Cartpole    & $234.3\pm92.5$    & $446.7\pm53.3$   & $393.4\pm55.5$   & $\bm{500.0\pm0.0}$~~         \\ 
~~Pendulum     & $-604.1\pm200.4$  & $-264.2\pm86.8$  & $-377.9\pm107.1$ & $\bm{-115.5\pm51.4}$~~   \\ 
~~Halfcheetah & $6055.6\pm843.1$  & $6719.4\pm753.8$ & $8561.5\pm775.3$ & $\bm{10794.8\pm520.2}$~~ \\ 
~~Hopper       & $1854.6\pm294.9$  & $1997.7\pm297.7$ & $2446.1\pm376.5$ & $\bm{3030.9\pm68.8}$~~  \\ 
~~Pong         & $14.6\pm2.5$      & $10.5\pm9.2$     & $18.8\pm2.2$     & $\bm{20.8\pm0.2}$~~      \\
~~Breakout     & $128.1\pm51.6$    & $255.6\pm47.8$   & $286.2\pm54.3$   & $\bm{336.5\pm42.1}$~~  \\ \bottomrule
\end{tabular}
}
\label{tab:performance}
\vspace{-1.8em}
\end{table}




\subsubsection{Linear speedup}
To answer the third question, we conduct experiments by varying the number of agents from 1 to 50.  
\cref{fig:agent_number} reveals a significant improvement in performance as the number of participating agents increases, which aligns well with our theoretical findings. That is, \alg{} adeptly controls the inter-agent gradient shift, thereby preventing variance accumulation even with an increasing number of agents involved.



\subsubsection{Impact of local steps}


To validate the impact of the number of local updates, denoted as $K$, we conduct experiments by varying its value from 1 to 50. The results, displayed in \cref{fig:local_steps}, show that the performance initially improves with an increasing value of $K$ and then starts to decline, consistent with our theoretical results (referring to the last term in \cref{eq:eq31}). The reason behind this trend is that a larger value of $K$ exacerbates the gradient shifts across different agents.

\subsubsection{Impact of batch sizes}
\cref{fig:batch_sizes} shows the impact of the number of collected trajectories in each updating step. 
As expected, when using a small value of $D$, the performance degrades dramatically, primarily because of the substantially high variance in the gradient estimator.

\section{Conclusion}

This paper introduces a new momentum-assisted federated policy optimization algorithm, namely \alg{}, to cope with the spatio-temporal non-stationarity of data distributions in FRL. We theoretically show that \alg{} offers the best communication and interaction complexities over the existing FRL methods, and provide extensive experiments to corroborate its superior performance over the baselines in continuous and high-dimensional environments. In future work, we will investigate offline/batch FRL approaches that can extract policies only from distributed \emph{static} data with no need to interact with environments. The authors have provided public access to their code at \href{https://codeocean.com/capsule/1418921/tree/v1}{https://codeocean.com/capsule/1418921/tree/v1}.



\section*{Acknowledgments}
This research was supported in part by NSFC under Grant No. 62341201, 62122095, 62072472, 62172445, 62302260, and 62202256, by the National Key R\&D Program of China under Grant No. 2022YFF0604502, by CPSF Grant 2023M731956, and by a grant from the Guoqiang Institute, Tsinghua University.

\newpage

\bibliographystyle{IEEEtran}
\bibliography{reference}

\end{document}